\PassOptionsToPackage{table}{xcolor}
\documentclass[11pt]{article}

\usepackage[final]{acl}

\usepackage{times}
\usepackage{latexsym}

\usepackage[T1]{fontenc}

\usepackage[utf8]{inputenc}

\usepackage{microtype}

\usepackage{inconsolata}

\usepackage{graphicx}

\usepackage{amsmath}
\usepackage{amsfonts}
\usepackage{amssymb}
\usepackage{wrapfig}
\usepackage{subcaption}
\usepackage{booktabs} %
\usepackage{multirow}
\usepackage{xcolor}
\definecolor{oursrow}{gray}{0.92}
\usepackage{listings}
\usepackage{hyperref}
\usepackage{xurl}

\usepackage{float}
\usepackage{amsthm}
\usepackage{enumitem}
\usepackage[capitalize,noabbrev]{cleveref}
\crefname{assumption}{Assumption}{Assumptions}

\setlist[itemize]{nosep,leftmargin=*,labelsep=0.4em}
\setlist[enumerate]{nosep,leftmargin=*,labelsep=0.4em}

\theoremstyle{plain}
\newtheorem{theorem}{Theorem}[section]

\newtheorem{lemma}[theorem]{Lemma}
\newtheorem{corollary}[theorem]{Corollary}
\theoremstyle{definition}

\newtheorem{assumption}[theorem]{Assumption}
\theoremstyle{remark}

\title{DSPA: Dynamic SAE Steering for Data-Efficient Preference Alignment}

\author{James Wedgwood \quad Aashiq Muhamed \quad Mona T. Diab \quad Virginia Smith \\
  Carnegie Mellon University \\
  \texttt{\{jwedgwoo, amuhamed, mdiab, smithv\}@cs.cmu.edu}}

\begin{document}

\maketitle

\begin{abstract}
\looseness=-1
Preference alignment is usually achieved by weight-updating training on preference data, which is costly and provides limited mechanistic visibility. We propose Dynamic SAE Steering for Preference Alignment (DSPA), an inference-time method that makes sparse autoencoder (SAE) steering prompt-conditional. From preference triples, DSPA computes a conditional-difference map linking prompt features to generation-control features; during decoding, it modifies only token-active latents, without base-model weight updates. Across Gemma-2-2B/9B and Qwen3-8B, DSPA improves MT-Bench and is competitive on AlpacaEval while preserving multiple-choice accuracy. Under restricted preference data, DSPA remains robust and can rival heavier alignment pipelines at a fraction of the compute. Finally, we audit the SAE features DSPA modifies, finding that preference directions are dominated by discourse and stylistic signals, and  provide theory clarifying the conditional-difference map estimate and when top-$k$ ablation is principled.\footnote{Code is available at \url{https://github.com/jtbwedgwood/dspa}.}
\end{abstract}

\section{Introduction}

\looseness=-1
Preference alignment is central to deploying large language models (LLMs), but common pipelines rely on weight-updating training on preference data (e.g., RLHF, DPO \citep{rafailov2023direct}), which can be costly and mechanistically opaque. Inference-time activation interventions \citep{zou2023representation,kong2024aligning,pham2024householder} are cheaper and reversible, but are often defined as dense, global directions; applying the same edit in every context can degrade open-ended generation and yield hard-to-audit tradeoffs \citep{wolf2024tradeoffs}.

\looseness=-1
Sparse autoencoders (SAEs) offer a more inspectable coordinate system: they decompose hidden states into sparse, named features that can be directly ablated or amplified \citep{rimsky2024steering,durmusevaluating}. However, preference alignment for open-ended generation is prompt-dependent. Static SAE feature sets can inject context-irrelevant signal and harm dialogue quality \citep{bayat2025steering}, motivating prompt-conditional feature selection and minimal, targeted edits.

We propose \textbf{Dynamic SAE Steering for Preference Alignment (DSPA)}, a prompt-conditional inference-time method that requires no base-model weight updates. DSPA uses an early--mid-layer \emph{input SAE} to featurize the prompt and a late-layer \emph{output SAE} as the intervention space most likely to influence generation \citep{Arad_2025}. Offline, we compute a sparse conditional-difference map $\mathbf{A}$ from preference triples that associates active prompt features with output features that differ between chosen and rejected responses. At inference, DSPA selects a small set of prompt features, scores output features via $\mathbf{A}$, and steers decoding by ablating and/or augmenting only those output latents that are active on the current token, reducing off-context perturbations relative to static steering.

\looseness=-1
We evaluate DSPA on Gemma-2-2B/9B and Qwen3-8B \citep{team2024gemma,yang2025qwen3technicalreport} using two open-ended generation benchmarks, MT-Bench \citep{zheng2023judgingllmasajudgemtbenchchatbot} and AlpacaEval \citep{alpaca_eval}, and multiple-choice benchmarks. DSPA improves MT-Bench across all three models; on AlpacaEval, it improves on Gemma-2-2B and Qwen3-8B, with modest changes on multiple-choice metrics. DSPA remains robust under restricted preference data (e.g., 250 samples; stable down to 100 on Gemma-2-2B). Because we intervene on named SAE features, we can audit what changes and find that preference directions are dominated by discourse and stylistic signals; our theory characterizes what $\mathbf{A}$ estimates and when top-$k$ ablation is a principled selection rule.
Our contributions include:

\begin{enumerate}
    \item We introduce DSPA, a prompt-conditional inference-time method that constructs a sparse association map from prompt SAE features (early/mid-layer) to generation-control SAE features (late-layer) via a conditional-difference map, and steers decoding by editing only token-active latents.

    \item We evaluate DSPA across Gemma-2-2B/9B and Qwen3-8B on open-ended generation benchmarks (MT-Bench, AlpacaEval) and multiple-choice benchmarks (MMLU, Arc-Easy, TruthfulQA-MC2, HellaSwag, Winogrande), showing competitive generation quality with minimal capability regression and robustness under preference-data restriction.
    
    \item We audit the SAE features DSPA modifies and find that preference gains are largely mediated by discourse and stylistic signals; we also provide a theoretical analysis of what the conditional-difference map estimates and why top-$k$ ablation is a principled selection rule.
\end{enumerate}

\section{Background and Related Work}
\label{sec:related}

\paragraph{Sparse Autoencoders.}
Superposition suggests many features share directions in activation space \citep{elhage2022toy}. Sparse autoencoders (SAEs) learn a sparse feature vector and a linear reconstruction of a layer's activations \citep{cunningham2023sparse,gao2024scaling}. Concretely, given activations $\mathbf{h}\in\mathbb{R}^{d_{\text{model}}}$, an SAE encodes $\mathbf{f}(\mathbf{h})=\sigma(W_{\text{enc}}\mathbf{h}+\mathbf{b}_{\text{enc}})$ and decodes $\hat{\mathbf{h}}(\mathbf{f})=W_{\text{dec}}\mathbf{f}+\mathbf{b}_{\text{dec}}$, trained to minimize reconstruction error with a sparsity penalty. We refer to the activation of feature $i$ as $f_i(\mathbf{h})$; sparse features enable auditable interventions by editing a small set of latents and decoding back to hidden-state space \citep{durmusevaluating}.

\paragraph{Preference Alignment.}
Most preference alignment methods update model weights using preference data (e.g., RLHF, DPO \citep{rafailov2023direct}), which can be effective but expensive and opaque. Inference-time interventions instead steer activations and are reversible. Representation engineering extracts linear control directions (mean-difference/PCA-style) and exposes alignment--capability tradeoffs \citep{zou2023representation,wolf2024tradeoffs}, while RAHF learns activity patterns from chosen/rejected data \citep{liu2024aligning}. Recent work refines steering with geometric and control-theoretic views \citep{kong2024aligning,pham2024householder} and studies category-specific directions for safety \citep{bhattacharjee2024towards}. DSPA builds on this inference-time line but makes steering prompt-conditional and operates in a sparse SAE basis for auditability.

\paragraph{SAEs for Preference Alignment.}
SAE steering has been used for refusal and safety control \citep{rimsky2024steering}, and static feature sets can work in constrained settings such as multiple-choice calibration \citep{bayat2025steering}. Feature selection remains central: effective steering depends on causal influence, not just activation frequency \citep{Arad_2025}. For open-ended generation, prior approaches often rely on learned steering policies or direct optimization in SAE space \citep{ferrao2025anatomy,bounhar2026yapo}. DSPA instead computes a map from prompt features to response-control features without training a steering policy or updating base-model weights.

\section{Dynamic SAE Steering for Preference Alignment}

\begin{figure*}[htbp]
  \centering
  \includegraphics[width=\linewidth]{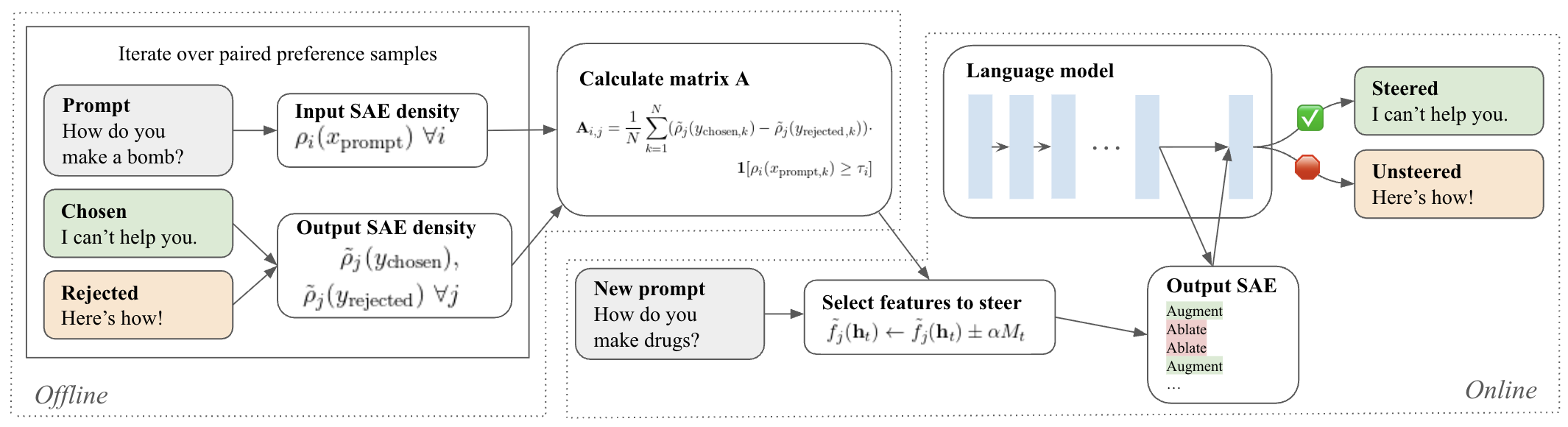}
  \caption{DSPA overview: offline conditional-difference map construction and prompt-conditional SAE steering at inference. Prompts and responses are illustrative. At decoding step $t$, $M_t$ is the largest active output-SAE latent and sets the scale of the edit.}
  \label{fig:method}
\end{figure*}

DSPA is a two-stage inference-time method. In an offline stage, we compute a sparse \emph{conditional-difference map} $\mathbf{A}$ from preference triples, associating prompt features with response-control features that differentiate chosen from rejected continuations. At inference, we use the prompt's active features to retrieve and rank a small set of output features and steer decoding by ablating and/or augmenting only token-active latents in the output SAE. Because interventions occur on sparse SAE features, the edited feature indices and token positions can be logged and audited; this does not imply that every edited latent has a reliable human-readable interpretation.
\paragraph{Relation to representation engineering.}
DSPA shares the use of chosen/rejected pairs to estimate preference directions \citep{zou2023representation,liu2024aligning}. If we ignore prompt conditioning and steer along a single dense direction, the construction reduces to global mean-difference steering. DSPA instead builds a prompt-conditional library of sparse templates (via input-feature gating) and applies \emph{token-conditional} edits in an SAE basis, modifying only latents that are active on the current token.

\subsection{Identifying Conditional Preference Features}
\label{sec:identifying}

\paragraph{Why two SAEs?}
The features that are useful to \emph{condition on} need not be the best ones to \emph{steer}. Following prior work suggesting that late-layer SAE features have stronger causal influence on generation \citep{Arad_2025}, we use an early--mid-layer \emph{input SAE} to featurize the prompt and a late-layer \emph{output SAE} as the intervention space; we validate this choice empirically (Section~\ref{sec:ablations}).
\paragraph{Notation.}
Let $\mathcal{D}=\{(x_k,y_k^+,y_k^-)\}_{k=1}^N$ be a dataset of preference triples, where $x$ is a prompt and $y^+,y^-$ are the chosen and rejected responses. Let $d_{\text{SAE}}$ be the SAE width (we assume the same width for input and output SAEs). We index input-SAE features by $i$ and output-SAE features by $j$, and we use a tilde to distinguish output-SAE quantities. Let $\mathbf{h}^{\text{in}}_t$ denote the LLM hidden state at token position $t$ at the input-SAE layer, and $\mathbf{h}^{\text{out}}_t$ the hidden state at the output-SAE layer. Then $f_i(\mathbf{h}^{\text{in}}_t)$ denotes input-SAE feature $i$ and $\tilde f_j(\mathbf{h}^{\text{out}}_t)$ denotes output-SAE feature $j$.

\paragraph{Activation densities.}
For a prompt $x=(x_1,\ldots,x_{T_x})$ and continuation $y=(y_1,\ldots,y_{T_y})$ (evaluated under teacher forcing after $x$), define
{\small
\begin{align*}
\rho_i(x) &:= \frac{1}{T_x}\sum_{t=1}^{T_x}\mathbf{1}\!\left[f_i(\mathbf{h}^{\text{in}}_t)>0\right],\\
\tilde\rho_j(x,y) &:= \frac{1}{T_y}\sum_{t=1}^{T_y}\mathbf{1}\!\left[\tilde f_j(\mathbf{h}^{\text{out}}_{T_x+t})>0\right].
\end{align*}
}
That is, $\rho_i(x)$ is the fraction of prompt tokens activating input feature $i$, and $\tilde\rho_j(x,y)$ is the fraction of response tokens activating output feature $j$ when processing the continuation $y$ after $x$.
\paragraph{Prompt gates.}
Choose a percentile $p\in(0,100)$ and let $\tau_i$ be the $p$th percentile of $\rho_i(x)$ over prompts in $\mathcal{D}$. Define the gate $g_i(x):=\mathbf{1}[\rho_i(x)\ge\tau_i]$ and the gate vector $g(x)\in\{0,1\}^{d_{\text{SAE}}}$.

\paragraph{Conditional-difference map.}
Let us define $\Delta\tilde{\boldsymbol\rho}(x,y^+,y^-)  :=\tilde{\boldsymbol\rho}(x,y^+)-\tilde{\boldsymbol\rho}(x,y^-)$, where $\tilde{\boldsymbol\rho}(x,y)\in[0,1]^{d_{\text{SAE}}}$ stacks $\tilde\rho_j(x,y)$. Write $\Delta\tilde{\boldsymbol\rho}_k:=\Delta\tilde{\boldsymbol\rho}(x_k,y_k^+,y_k^-)$ and compute $\mathbf{A}:=\frac{1}{N}\sum_{k=1}^N g(x_k)\,\Delta\tilde{\boldsymbol\rho}_k^\top$, where $\mathbf{A}\in\mathbb{R}^{d_{\text{SAE}}\times d_{\text{SAE}}}$.

Entrywise, $\mathbf{A}_{i,j}$ averages the chosen-minus-rejected activation-density difference for output feature $j$ over examples where prompt gate $i$ is active, so $\mathbf{A}_{i,j}$ estimates $\mathbb{E}[g_i(x)\,\Delta\tilde\rho_j]$. Note that $\mathbf{A}$ is not a learned reward model or steering policy; it is a fixed association map requiring no gradient-based optimization.
\paragraph{Sparsification.}
Although $\mathbf{A}$ has $d_{\text{SAE}}^2$ entries, it is sparse in practice. We zero out entries below a conservative threshold, reducing memory usage by 99.8\% (Appendix~\ref{app:base_vs_custom}).

\subsection{Inference-Time Intervention}

\paragraph{Prompt feature selection.}
Given a new prompt $x$, we compute $\rho_i(x)$ for all input features and select the top $k_{\text{prompt}}$ features by density, $S_{\text{prompt}}(x):=\text{top-}k_{\text{prompt}}\{\rho_i(x)\}$. Let $\hat g(x):=\mathbf{1}_{S_{\text{prompt}}(x)}\in\{0,1\}^{d_{\text{SAE}}}$, where $\mathbf{1}_S$ denotes the indicator vector of a set $S$.
\paragraph{Scoring output features.}
We score output features using the conditional-difference map, $\mathbf{s}(x):=\mathbf{A}^\top \hat g(x)$. We then select $k_{\text{diff}}$ features to augment and/or ablate: $S_{\text{augment}}(x):=\text{top-}k_{\text{diff}}\{\mathbf{s}_j(x)\}$ and $S_{\text{ablate}}(x):=\text{bottom-}k_{\text{diff}}\{\mathbf{s}_j(x)\}$.
\paragraph{Token-conditional intervention.}
$S_{\text{augment}}(x)$ and $S_{\text{ablate}}(x)$ specify \emph{which} output features to steer for prompt $x$; token-conditional steering edits only features active on the current token, avoiding off-context activations.

At token position $t$, we compute output-SAE latents $\tilde{\mathbf{f}}_t=\tilde{\mathbf{f}}(\mathbf{h}^{\text{out}}_t)$ and scale the step size by $M_t:=\smash[b]{\max_{j\in[d_{\text{SAE}}]} \tilde f_{t,j}}$. For $j\in S_{\text{augment}}(x)$ with $\tilde f_{t,j}>0$, we set $\tilde f_{t,j}\leftarrow \tilde f_{t,j}+\alpha M_t$; for $j\in S_{\text{ablate}}(x)$ with $\tilde f_{t,j}>0$, we set $\tilde f_{t,j}\leftarrow \max\{\tilde f_{t,j}-\alpha M_t,0\}$, yielding edited latents $\tilde{\mathbf{f}}'_t$. Let $\Delta\tilde{\mathbf{f}}_t:=\smash[b]{\tilde{\mathbf{f}}'_t-\tilde{\mathbf{f}}_t}$; we apply the residual edit by $\mathbf{h}^{\text{out}}_t \leftarrow \mathbf{h}^{\text{out}}_t + \smash[b]{\tilde W_{\text{dec}}\,\Delta\tilde{\mathbf{f}}_t}$. Because our SAEs use ReLU activations, this edits only already-active latents and clamps ablations at zero. In our main experiments we use ablation-only steering unless otherwise stated.

\subsection{Theoretical Analysis}
\label{sec:theory}

The preceding procedure is heuristic; we now provide formal justification for why the conditional-difference map recovers meaningful preference directions and why top-$k$ selection is principled. Proofs appear in Appendix~\ref{app:theory}.
\paragraph{Setup.}
We use the notation from Section~\ref{sec:identifying}: $x$ is the prompt and $y^+,y^-$ are the chosen/rejected responses. For a prompt $x$, let $g(x) \in \{0,1\}^{d_{\text{SAE}}}$ denote the \emph{prompt gate vector} with $g_i(x) := \mathbf{1}[\rho_i(x) \geq \tau_i]$, where $\rho_i(x)$ is the input-SAE activation density and $\tau_i$ is a fixed threshold. For a response $y$ generated after $x$, let $\tilde{\boldsymbol\rho}(y) \in [0,1]^{d_{\text{SAE}}}$ denote output-SAE activation densities over the response segment (equivalently $\tilde{\boldsymbol\rho}(x,y)$ in Section~\ref{sec:identifying}, suppressing $x$ for brevity). Given a preference triple $(x, y^+, y^-)$, define $\Delta\tilde{\boldsymbol\rho} := \tilde{\boldsymbol\rho}(y^+) - \tilde{\boldsymbol\rho}(y^-)$. The population analogue of our conditional-difference map is $\mathbf{A} := \mathbb{E}[g(x)\,\Delta\tilde{\boldsymbol\rho}^\top]$.

\begin{assumption}[Shared-covariance mean shift]
\label{ass:lda}
There exist a positive semidefinite matrix $\Sigma$, a scalar $c>0$, and a prompt-dependent vector $\boldsymbol\beta(x)$ such that, conditional on $x$, $\mathbb{E}[\Delta\tilde{\boldsymbol\rho} \mid x] = c\,\Sigma\boldsymbol\beta(x)$.
\end{assumption}

\begin{assumption}[Additive gating]
\label{ass:gating}
There exists a matrix $B \in \mathbb{R}^{d_{\text{SAE}} \times d_{\text{SAE}}}$ such that $\boldsymbol\beta(x) = B\,g(x)$. Let $M := \mathbb{E}[g(x)\,g(x)^\top]$ denote the gate Gram matrix.
\end{assumption}

\begin{theorem}[Factorization of $\mathbf{A}$]
\label{thm:identify_exact}
Under \cref{ass:lda,ass:gating}, $\mathbf{A}^\top = c\,\Sigma\,B\,M$. Consequently, rows of $\mathbf{A}$ mix multiple preference templates when gates co-activate. On subspaces where $M$ is invertible, $\Sigma^{-1}\mathbf{A}^\top M^{-1} \propto B$.
\end{theorem}

The mixing is controlled when co-activation probabilities are small:
\begin{corollary}[Weak co-activation bound]
\label{cor:weak_coact}
Fix gate $i$ and let $\pi_{i'|i} := \Pr(g_{i'}=1 \mid g_i=1)$. Then $\|\mathbb{E}[\Delta\tilde{\boldsymbol\rho} \mid g_i=1] - c\,\Sigma\boldsymbol\beta^{(i)}\|_2 \leq c\,\|\Sigma\|_{2\to 2} \sum_{i'\neq i} \pi_{i'|i} \|\boldsymbol\beta^{(i')}\|_2$.
\end{corollary}
\paragraph{Connection to DSPA.}
At inference, DSPA computes $\mathbf{s}(x) := \mathbf{A}^\top \hat{g}(x)$ where $\hat{g}(x)$ is a truncated indicator over active prompt features. By \cref{thm:identify_exact}, $\mathbf{s}(x) = c\,\Sigma\,B\,M\,\hat{g}(x)$, a mixed estimate of the prompt-conditional preference direction. When $M \approx \mathrm{diag}(p)$, this reduces to a reweighted sum of per-gate templates. Under a linear utility model, selecting output features with the most negative scores for ablation is optimal for a fixed budget (\cref{thm:topk}; Appendix~\ref{app:theory}).
\paragraph{Implications.}
These results justify several design choices. First, the factorization $\mathbf{A}^\top = c\,\Sigma\,B\,M$ shows that $\mathbf{A}$ is a principled estimator of prompt-conditional preference directions even though it is constructed by simple averaging, not optimization. Second, \cref{cor:weak_coact} explains why truncation to the top $k_\text{prompt}$ gates is reasonable: when prompt features activate approximately independently, each row of $\mathbf{A}$ approximates a clean per-gate preference template with bounded contamination from co-activating gates. This also clarifies why static (non-prompt-conditional) steering is brittle for open-ended generation: it ignores co-activation structure entirely. Third, \cref{thm:topk} justifies top-$k$ ablation as the optimal selection rule under a linearized utility model. Finally, the finite-sample concentration result (\cref{lem:concentration}, Appendix~\ref{app:theory}) connects to our data-efficiency findings: frequently activated gates yield stable conditional estimates even with limited data, while rare gates are noisy, motivating our percentile thresholding and conservative sparsification of $\mathbf{A}$.

\section{Experiments and Results}

\subsection{Models, Data, and SAEs}
\label{sec:experimental-setup}

\begin{figure*}[htbp]
  \centering
  \includegraphics[width=0.9\linewidth]{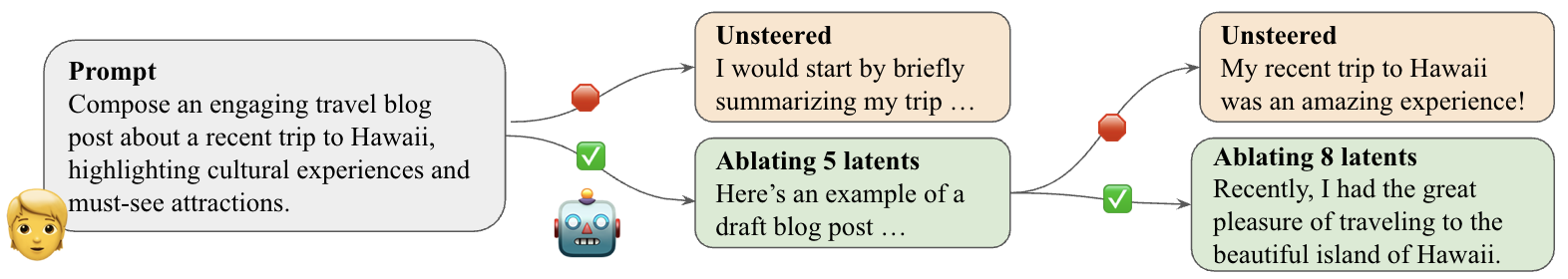}
  \caption{MT-Bench (turn 1) example on Gemma-2-9B: Base Model vs. DSPA.}
  \label{fig:example}
\end{figure*}

Figure~\ref{fig:example} shows a qualitative example illustrating how DSPA changes open-ended responses.
\paragraph{Training Data} We use a binarized version of UltraFeedback \citep{cui2024ultrafeedbackboostinglanguagemodels} (61K prompts with chosen/rejected responses).\footnote{HF: \texttt{argilla/ultrafeedback-binarized-preferences-\allowbreak{}cleaned}.} We use these preference pairs to construct $\mathbf{A}$ and, where applicable, to train weight-updating baselines (DPO/RAHF) and derive inference-time steering controls (RepE direction; Static-SAE feature set). Appendix~\ref{app:preference_sources} tests whether the map transfers to two non-UltraFeedback preference sources.
\paragraph{Models} Following common alignment setups \citep{ouyang2022traininglanguagemodelsfollow,liu2024aligning}, we start from supervised fine-tuned (SFT) base models on HH-RLHF \citep{bai2022traininghelpfulharmlessassistant} for Gemma-2-2B/9B and Qwen3-8B \citep{team2024gemma, yang2025qwen3technicalreport}.\footnote{HF: \texttt{Anthropic/hh-rlhf}.} We use these architectures because high-quality open SAEs are available; ``Base Model'' refers to the HH-RLHF SFT checkpoint.
\paragraph{SAEs} We use SAEs fine-tuned on HH-RLHF (Appendix~\ref{app:interp}), which yields better-aligned features. For Gemma, we fine-tune Gemma Scope JumpReLU SAEs \citep{lieberum2024gemmascopeopensparse,rajamanoharan2024jumping}; for Qwen, we use released BatchTopK SAEs \citep{bussmann2024batchtopksparseautoencoders, karvonen2025qwen3saes}.\footnote{HF: \texttt{adamkarvonen/qwen3-8b-saes}.} Further details on SAE fine-tuning may be found in Appendix~\ref{app:base_vs_custom}.

\paragraph{Baselines}
\looseness=-1
We include the HH-RLHF supervised fine-tuned starting point (``Base Model'') as a reference, and compare DSPA to four baselines: DPO \citep{rafailov2023direct} fine-tuned on UltraFeedback chosen/rejected pairs \citep{cui2024ultrafeedbackboostinglanguagemodels}, a dense representation-engineering control direction (RepE) estimated from chosen/rejected data \citep{zou2023representation}, a static SAE steering baseline that edits a fixed globally selected feature set (Static-SAE) \citep{ferrao2025anatomy,bayat2025steering}, and an instruction-style prompt prefix (Prompt Eng) \citep{liu2024aligning,wu2025axbenchsteeringllmssimple}.
DPO updates base-model weights, whereas DSPA/RepE/Static-SAE are inference-time activation interventions. In Section~\ref{sec:data_efficiency}, we additionally compare against RAHF-SCIT \citep{liu2024aligning}, a two-stage fine-tuning pipeline, under a restricted-data setting.
Full hyperparameters and implementation details are in Appendix~\ref{app:baselines}.

\subsection{Main Results: Open-Ended and Multiple-Choice Benchmarks}

\begin{table*}[t]
\small
\caption{Open-ended and multiple-choice results. Open-ended scores are averaged over three generation seeds; multiple-choice scores are deterministic. Bold = best; underlined = second best within each model group. Arc = Arc-Easy; TQA = TruthfulQA-MC2. See Appendix~\ref{app:experiment}.}
\label{tab:gemma-combined}
\centering
\setlength{\tabcolsep}{5pt}
\begin{tabular}{@{}cl ccccccc@{}}
\toprule
& & \multicolumn{2}{c}{\textit{Open-ended}} & \multicolumn{5}{c}{\textit{Multiple-choice}} \\
\cmidrule(lr){3-4} \cmidrule(lr){5-9}
& & \bf MT-Bench & \bf AlpacaEval & \bf MMLU & \bf Arc & \bf TQA & \bf HellaSwag & \bf Winogrande \\
\midrule
\multirow{6}{*}{\rotatebox{90}{\footnotesize Gemma-2-2B}}
& \cellcolor{oursrow} DSPA (ours)   & \cellcolor{oursrow}\bf 4.39 & \cellcolor{oursrow}\textbf{19.30} & \cellcolor{oursrow}\bf 48.1 & \cellcolor{oursrow}80.4 & \cellcolor{oursrow}40.4 & \cellcolor{oursrow}74.2 & \cellcolor{oursrow}65.0 \\
& DPO           & \underline{4.36} & 15.82 & 47.9 & \bf 81.0 & 39.8 & \bf 74.3 & 65.2 \\
& RepE          & 4.32 & 16.94 & 48.0 & 80.3 & 40.3 & \bf 74.3 & 65.1 \\
& Static-SAE    & 3.49 & 9.33  & 48.0 & 80.6 & 39.5  & 74.2 & \bf 65.7 \\
& Prompt Eng    & 3.84 & 15.28 & 47.8 & 80.8 & \bf 41.0 & 74.0 & 64.3 \\
& Base Model    & 3.91 & \underline{17.54} & 47.9 & 80.4 & 40.2 & 74.2 & \bf 65.7 \\
\midrule
\multirow{6}{*}{\rotatebox{90}{\footnotesize Gemma-2-9B}}
& \cellcolor{oursrow} DSPA (ours)   & \cellcolor{oursrow}\underline{5.02} & \cellcolor{oursrow}20.00 & \cellcolor{oursrow}\bf 60.6 & \cellcolor{oursrow}\bf 83.8 & \cellcolor{oursrow}45.4 & \cellcolor{oursrow}77.3 & \cellcolor{oursrow}\bf 72.5 \\
& DPO           & \textbf{5.05} & 17.39 & 60.4 & 83.6 & 44.1 & 77.3 & \bf 72.5 \\
& RepE          & 4.36 & 18.88 & 60.4 & \bf 83.8 & 45.5 & 77.5 & \bf 72.5 \\
& Static-SAE    & 4.33 & 13.56 & \bf 60.6 & 83.6 & 44.1 & \bf 77.6 & 72.5 \\
& Prompt Eng    & 4.60 & \textbf{22.51} & 60.0 & 83.7 & \bf 47.2 & 77.4 & 71.9 \\
& Base Model    & 4.39 & \underline{21.12} & 60.0 & 83.7 & 45.5 & 77.4 & 71.9 \\
\midrule
\multirow{6}{*}{\rotatebox{90}{\footnotesize Qwen3-8B}}
& \cellcolor{oursrow} DSPA (ours)   & \cellcolor{oursrow}\underline{4.86} & \cellcolor{oursrow}\underline{17.83} & \cellcolor{oursrow}71.2 & \cellcolor{oursrow}84.3 & \cellcolor{oursrow}42.4 & \cellcolor{oursrow}72.6 & \cellcolor{oursrow}73.8 \\
& DPO           & \textbf{5.52} & \bf 24.60 & \bf 72.6 & \bf 85.5 & \bf 45.9 & \bf 73.0 & 73.4 \\
& RepE          & 4.48 & 12.95 & 71.2 & 84.4 & 41.0 & 71.7 & 73.1 \\
& Static-SAE    & 3.24 & 8.58 & 71.2 & 84.1 & 42.4 & 72.7 & \bf 74.0 \\
& Prompt Eng    & 3.88 & 11.74 & 71.0 & 84.4 & 42.6 & 71.7 & 72.8 \\
& Base Model    & 4.12 & 13.47 & 71.1 & 84.4 & 40.9 & 71.6 & 72.0 \\
\bottomrule
\end{tabular}
\end{table*}

We evaluate open-ended generation with MT-Bench \citep{zheng2023judgingllmasajudgemtbenchchatbot} and AlpacaEval \citep{alpaca_eval}, both using LLM-as-a-judge protocols. We use GPT-4o as the MT-Bench judge; for AlpacaEval, we use the \texttt{alpaca\_eval\_llama3\_70b\_fn} annotator, so absolute AlpacaEval scores are not directly comparable to leaderboards that use different judges. To monitor general capability, we report multiple-choice benchmarks (MMLU, Arc-Easy, TruthfulQA-MC2, HellaSwag, Winogrande) via the Language Model Evaluation Harness \citep{eval-harness} using standard settings (Appendix~\ref{app:experiment}).

Table~\ref{tab:gemma-combined} summarizes results (open-ended scores averaged over three generation seeds; small differences should be interpreted cautiously). On MT-Bench, DSPA improves over the Base Model and over other inference-time baselines for all three models, and it matches or slightly exceeds DPO on Gemma-2-2B without weight updates. DPO is slightly higher on Gemma-2-9B MT-Bench and substantially stronger on both open-ended metrics for Qwen3-8B. On AlpacaEval, DSPA improves over the Base Model on Gemma-2-2B and Qwen3-8B, but underperforms both Prompt Eng and the Base Model on Gemma-2-9B. Thus, the results support a competitive lightweight alternative, not uniform dominance over DPO. Our DPO runs used standard DPO/TRL-style settings and a limited sweep rather than extensive per-model tuning; stronger DPO results may be possible at additional tuning cost (Appendix~\ref{app:baselines}). Static-SAE consistently degrades open-ended scores, supporting the need for prompt-conditional feature selection rather than a fixed global feature set. Multiple-choice metrics change only modestly across methods, suggesting limited capability regression. In Appendix~\ref{app:mt_bench}, we further break down MT-Bench scores by category.

Because both open-ended benchmarks use automated judges, style sensitivity remains a limitation. A simple verbosity explanation is not supported by our diagnostics: DSPA's Gemma AlpacaEval outputs are slightly shorter than the Base Model outputs, and Gemma-2-9B MT-Bench gains also occur in Coding, Math, and STEM rather than only Roleplay or Writing (Appendix~\ref{app:mt_bench}). These checks do not replace controlled human evaluation or length- and style-controlled judging.

\subsection{Anatomy of Preference in SAE Space}
\label{sec:interp}

\begin{table*}[t]
\caption{Top five ablate/augment-set features, with \texttt{gpt-5-mini} interpretations.}
\label{tab:interp1}
\small
\centering
\begin{tabular}{@{}clp{10.5cm}@{}}
\toprule
& \bf Feature & \bf Interpretation \\
\midrule
\multirow{5}{*}{\rotatebox{90}{\footnotesize Ablate}}
& 11569 & Filler phrases used to smooth conversation and request clarification \\
& 10776 & Questions or statements involving illegal or clearly harmful activities \\
& 6345 & Polite opening phrases that frame what follows \\
& 11287 & First-person statements expressing personal views or experiences \\
& 12550 & Short connective phrases linking ideas across clauses \\
\midrule
\multirow{5}{*}{\rotatebox{90}{\footnotesize Augment}}
& 141 & Common function words and grammatical connectors without semantic content \\
& 9825 & Clarification and intent-seeking phrases linking questions and follow-ups \\
& 10075 & Topic-naming or subject-introducing words and phrases \\
& 1030 & Short list items and connectors enumerating entities or phrases \\
& 11031 & Affirmative explanatory openings that introduce a helpful continuation \\
\bottomrule
\end{tabular}
\end{table*}

DSPA exposes a more inspectable intervention trace than dense steering and fine-tuning methods: the sparse feature indices edited at each token can be recorded directly. Human-readable descriptions of those features remain approximate, however, and an edited feature should not be treated as a complete causal explanation of the generated text. Appendix~\ref{app:token_audit} provides an illustrative token-level audit. The aggregate analysis below is consistent with the theoretical expectation from Section~\ref{sec:theory}: because $\mathbf{A}$ averages over many prompts, the features that emerge most strongly are high-frequency, broadly reusable interactional signals rather than prompt-specific content.

We approximate frequently steered output features by ranking columns $j$ of $\mathbf{A}$ by $\sum_i \mathbf{A}_{i,j}$: the top/bottom 50 form ``augment''/``ablate'' sets (Gemma-2-9B, $\ell_{\text{output}} = 39$). We find that these account for most features steered in practice: on the first turn of MT-Bench, the ablated features form a strict subset of our ablate set. In an auxiliary augment+ablate pass used only for this coverage check, all but four augmented features belong to the augment set.

We generate an explanation for each of these features via an interpretability pipeline using \texttt{gpt-5-mini} 
as a judge (see Appendix~\ref{app:interp-prompts}). As these interpretations are LLM-generated and have not been systematically validated by humans, they should be treated as approximate. Descriptions of the top five augment and ablate set features, ordered by magnitude $\left|\sum_i \mathbf{A}_{i,j}\right|$, appear in Table~\ref{tab:interp1}.
We also prompt \texttt{gpt-5-mini} to categorize each feature into one of four categories:
\begin{enumerate}
    \item \textbf{Content:} Topic-bearing or referential language with substantive semantic meaning
    \item \textbf{Discourse:} Language that manages conversational flow or interaction rather than content
    \item \textbf{Grammatical:} Grammatical function words with minimal standalone meaning
    \item \textbf{Structure:} Tokens whose primary role is positional or structural rather than semantic
\end{enumerate}

\begin{table}[t]
\caption{Category breakdown for ablate/augment sets (50 each) vs.\ 50 random SAE features.}
\label{tab:interp2}
\small
\centering
\begin{tabular}{@{}lccc@{}}
\toprule
\bf Category & \bf Ablate & \bf Augment & \bf Random \\
\midrule
Content & 5 & 7 & 18 \\
\rowcolor{oursrow} Discourse & 34 & 28 & 15 \\
Grammatical & 10 & 10 & 11 \\
Structure & 1 & 5 & 6 \\
\bottomrule
\end{tabular}
\end{table}

Table~\ref{tab:interp2} shows the breakdown by category for augment and ablate sets, compared to a sample of 50 random features as a baseline. Preference directions are dominated by discourse and style markers: both the augment and ablate sets contain far more Discourse and far fewer Content features than the random set. This is consistent with the view that preference-relevant signals for open-ended generation are primarily interactional --- tone, hedging, clarification --- rather than about specific concepts. Notably, relatively few features in either set are Structure features, indicating that DSPA relies on conversational signals rather than formatting like lists or code. This contrasts with another recent SAE steering method, FSRL \citep{ferrao2025anatomy}, which primarily relies on features related to ``structural presentation elements.'' Safety-specific signal is limited: the only directly safety-related feature is the second-most ablated one, index 10776, representing \textit{Questions or statements involving illegal or clearly harmful activities}.

\subsection{Data Efficiency and Compute Cost}
\label{sec:data_efficiency}

\begin{table*}[t]
\small
\caption{Restricted-data results ($N{=}250$ preference triples). Open-ended scores are averaged over three generation seeds; multiple-choice scores are deterministic. Bold = best; underlined = second best within each model. 2B = Gemma-2-2B; 9B = Gemma-2-9B; Qwen = Qwen3-8B.}
\label{tab:restricted-data}
\centering
\setlength{\tabcolsep}{5pt}
\begin{tabular}{@{}cl ccccccc@{}}
\toprule
& & \multicolumn{2}{c}{\textit{Open-ended}} & \multicolumn{5}{c}{\textit{Multiple-choice}} \\
\cmidrule(lr){3-4} \cmidrule(lr){5-9}
& & \bf MT-Bench & \bf AlpacaEval & \bf MMLU & \bf Arc & \bf TQA & \bf HellaSwag & \bf Winogrande \\
\midrule
\multirow{4}{*}{\rotatebox{90}{\footnotesize 2B}}
& \cellcolor{oursrow} DSPA & \cellcolor{oursrow}\underline{3.95} & \cellcolor{oursrow}\underline{16.81} & \cellcolor{oursrow}\bf 48.0 & \cellcolor{oursrow}80.4 & \cellcolor{oursrow}40.4 & \cellcolor{oursrow}74.2 & \cellcolor{oursrow}65.0 \\
& RAHF-SCIT & 3.71 & \textbf{19.13} & 45.8 & 78.7 & \bf 40.6 & 68.8 & \bf 66.3 \\
& DPO & 3.91 & 15.42 & 47.9 & \bf 80.8 & 39.6 & 74.0 & 64.6 \\
& RepE & \textbf{3.97} & 15.80 & \bf 48.0 & 80.5 & 40.3 & \bf 74.3 & 65.0 \\
\midrule
\multirow{4}{*}{\rotatebox{90}{\footnotesize 9B}}
& \cellcolor{oursrow} DSPA & \cellcolor{oursrow}\textbf{4.82} & \cellcolor{oursrow}\underline{21.27} & \cellcolor{oursrow}\bf 60.9 & \cellcolor{oursrow}83.8 & \cellcolor{oursrow}45.4 & \cellcolor{oursrow}77.5 & \cellcolor{oursrow}72.5 \\
& RAHF-SCIT & \underline{4.63} & \textbf{22.64} & 58.8 & 82.6 & 45.4 & 73.6 & 72.4 \\
& DPO & 4.10 & 16.52 & 60.5 & \bf 84.6 & 44.8 & \bf 77.6 & \bf 72.8 \\
& RepE & 4.47 & 20.62 & 60.5 & 83.9 & \bf 45.5 & 77.3 & 72.4 \\
\midrule
\multirow{4}{*}{\rotatebox{90}{\footnotesize Qwen}}
& \cellcolor{oursrow} DSPA & \cellcolor{oursrow}\textbf{4.69} & \cellcolor{oursrow}\underline{18.20} & \cellcolor{oursrow}71.1 & \cellcolor{oursrow}84.3 & \cellcolor{oursrow}42.3 & \cellcolor{oursrow}\textbf{72.5} & \cellcolor{oursrow}\textbf{73.7} \\
& RAHF-SCIT & \underline{4.52} & \bf 23.48 & \bf 71.3 & \bf 84.6 & \bf 48.2 & 72.4 & 73.0 \\
& DPO & 4.34 & 16.84 & 71.3 & 84.5 & 41.1 & 71.7 & 72.4 \\
& RepE & 4.51 & 14.34 & 71.2 & 84.2 & 41.0 & 71.6 & 72.1 \\
\bottomrule
\end{tabular}
\end{table*}

\begin{figure}
  \centering
  \includegraphics[width=0.9\linewidth]{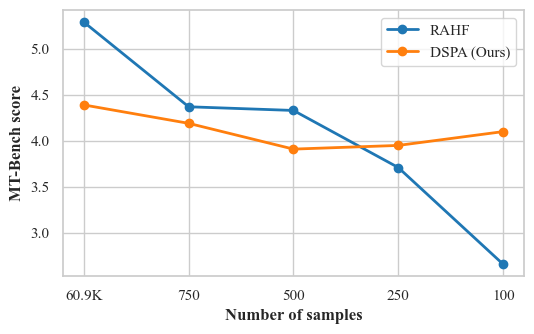}
\caption{MT-Bench under data restriction (Gemma-2-2B): DSPA vs.\ RAHF-SCIT.}
  \label{fig:mt_bench}
\end{figure}

To study data efficiency, we subsample $N{=}250$ preference triples from UltraFeedback and recompute $\mathbf{A}$; for weight-updating methods, we train on the same subset. We compare against DPO and RepE, as well as RAHF-SCIT \citep{liu2024aligning}, a strong but computationally heavy two-stage fine-tuning pipeline. Using standard FLOP accounting \citep{brown2020languagemodelsfewshotlearners, chowdhery2022palmscalinglanguagemodeling}, we estimate that the alignment stage of RAHF-SCIT is approximately $4.47\times$ as compute-intensive as DSPA's $\mathbf{A}$ construction. In practice, on Gemma-2-9B on a single Nvidia H200, RAHF-SCIT takes $11.5\times$ longer wall-clock with peak memory 140.8\,GB versus 33.1\,GB. DPO is also substantially more expensive at full data: our Qwen3-8B run took nearly 12 hours on four L40S GPUs, while DSPA construction completed in roughly one-fifteenth the wall-clock time on one L40S (Appendix~\ref{app:compute_cost}). These comparisons count only additional alignment computation beyond the shared SFT base.

Table~\ref{tab:restricted-data} reports results at $N{=}250$. In this setting, DSPA achieves the highest MT-Bench scores on Gemma-2-9B and Qwen3-8B and is within 0.02 of the best method on Gemma-2-2B. Both DSPA and RAHF-SCIT outperform DPO and RepE under data restriction. Figure~\ref{fig:mt_bench} further sweeps $N$ for Gemma-2-2B: DSPA remains stable even with as few as $N{=}100$ preference triples, while RAHF-SCIT's score declines sharply. This robustness is consistent with the concentration analysis in Appendix~\ref{app:theory}: because the most frequently steered features (Section~\ref{sec:interp}) are broadly reusable discourse signals, their conditional estimates in $\mathbf{A}$ stabilize quickly even with limited data.

\subsection{Ablations}
\label{sec:ablations}

In the preceding experiments, we selected a consistent set of hyperparameters that yields strong results across all base models. In particular, we use $k_{\text{prompt}} = 32, k_{\text{diff}} = 16, \alpha = 0.2$, and we ablate dispreferred SAE features as opposed to augmenting preferred ones. We set $\ell_{\text{input}}=7, \ell_{\text{output}}=24$ for Gemma-2-2B, $\ell_{\text{input}}=9, \ell_{\text{output}}=39$ for Gemma-2-9B, and $\ell_{\text{input}} = 9, \ell_{\text{output}} = 18$ for Qwen3-8B. Moreover, as discussed in Section~\ref{sec:experimental-setup}, we use custom SAEs fine-tuned on HH-RLHF. In this section, we justify our choice of input and output layer, mode (ablate vs. augment), and the use of custom SAEs in a series of ablation experiments run on the two Gemma models.

For ablations, we use MT-Bench judged by GPT-OSS-120B (much cheaper than GPT-4o); these scores are not directly comparable to our main MT-Bench results and are generally lower.
\begin{figure}[htb]
  \centering
  \includegraphics[width=\linewidth]{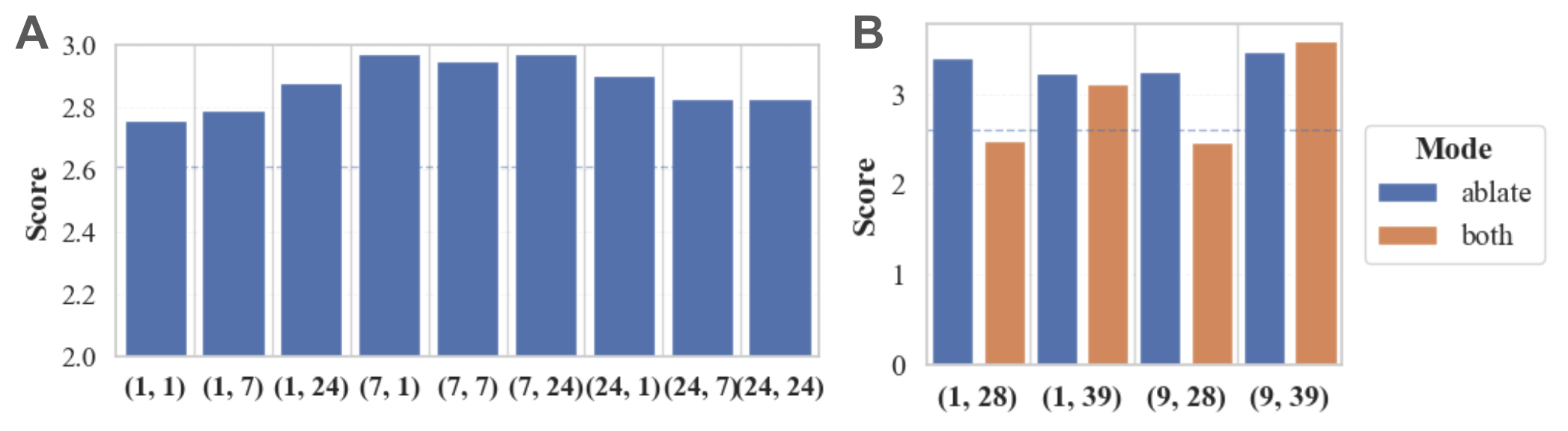}
  \caption{\textbf{A.} Gemma-2-2B score vs.\ $(\ell_{\text{input}}, \ell_{\text{output}})$ (ablate only). \textbf{B.} Gemma-2-9B score vs.\ $(\ell_{\text{input}}, \ell_{\text{output}})$ (ablate only; augment+ablate). Dashed = Base Model.}
  \label{fig:ablations_layer}
\end{figure}

\paragraph{Which layers matter?} We examine the effect of different choices of $\ell_{\text{input}}$ and $\ell_{\text{output}}$ in Figure~\ref{fig:ablations_layer}A for Gemma-2-2B and in Figure~\ref{fig:ablations_layer}B for Gemma-2-9B. We observe that for Gemma-2-2B $\ell_{\text{input}}$ has a stronger impact on score than $\ell_{\text{output}}$, with $\ell_{\text{input}}=7$ outperforming the other input layers surveyed. $\ell_{\text{input}}=7, \ell_{\text{output}}=24$ has marginally better performance than other output layers choices; we also prefer a late output layer because late-layer SAE features are more likely to be interpretable.

For Gemma-2-9B, we observe stronger dependence on both input and output layer, particularly when both augmenting and ablating features. The best score is achieved in both cases for $\ell_{\text{input}}=9, \ell_{\text{output}}=39$. This is consistent with our choice of layers for Gemma-2-2B: in both cases we use an early--mid input layer and a late output layer.

\begin{figure}[htb]
  \centering
  \includegraphics[width=\linewidth]{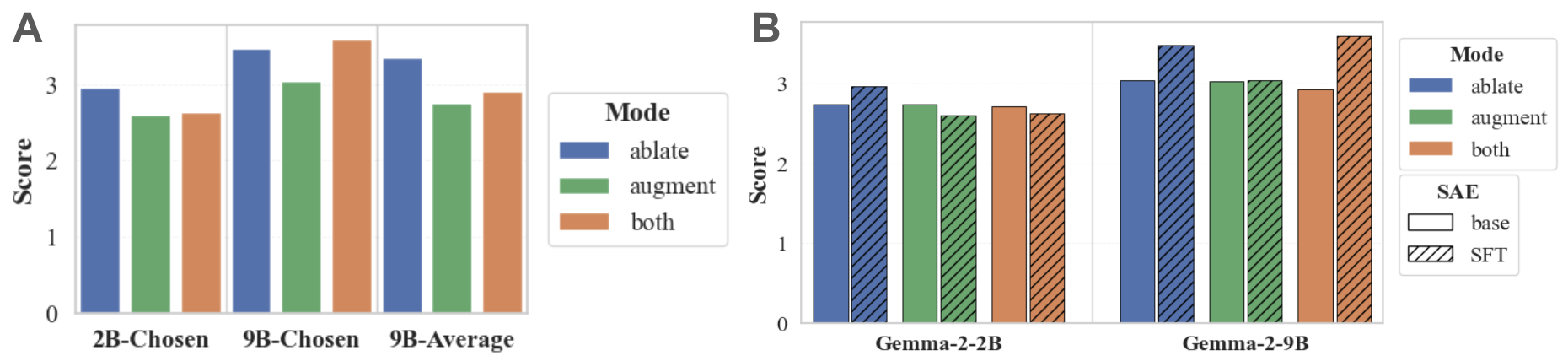}
  \caption{\textbf{A.} Score by steering mode (ablate/augment/both). 2B-Chosen and 9B-Chosen use the best $(\ell_{\text{input}}, \ell_{\text{output}})$ per model; 9B-Average averages over the layer grid in Figure~\ref{fig:ablations_layer}B. \textbf{B.} Base Gemma Scope SAE vs.\ HH-RLHF fine-tuned SAE.}
  \label{fig:ablations_mode_base}
\end{figure}

\paragraph{Ablate, augment, or both?} A key decision point in applying DSPA is whether to augment desirable features, ablate undesirable ones, or do both; Figure~\ref{fig:ablations_mode_base}A summarizes the impact of this choice. We find that for both models only augmenting features leads to subpar results. For Gemma-2-2B, only ablating features is consistently better than either of the other modes. For Gemma-2-9B, under certain layer combinations including our main choice of $\ell_{\text{input}}=9, \ell_{\text{output}}=39$, both augmenting and ablating yields a better score than ablating only, but this is not consistent across layer choices. For consistency, we only ablate features throughout.

\paragraph{Does fine-tuning the SAE help?} In Figure~\ref{fig:ablations_mode_base}B, we assess the impact of steering the base Gemma Scope SAEs \citep{lieberum2024gemmascopeopensparse} versus the custom SAEs fine-tuned on HH-RLHF which we have used throughout. We report results across both models and all three modes, keeping all other parameters constant. The public and custom SAEs tie when averaged across steering modes on Gemma-2-2B (2.73 vs.\ 2.73) and are nearly tied for augment-only Gemma-2-9B (3.03 vs.\ 3.05). Custom SAEs are materially better for Gemma-2-9B under ablate-only (3.48 vs.\ 3.04) and augment+ablate (3.60 vs.\ 2.93). Thus off-the-shelf SAEs can be usable when available, while preference-domain fine-tuning helps substantially in some settings. See Appendix~\ref{app:base_vs_custom} for more details.

\begin{figure}[htb]
  \centering
  \includegraphics[width=\linewidth]{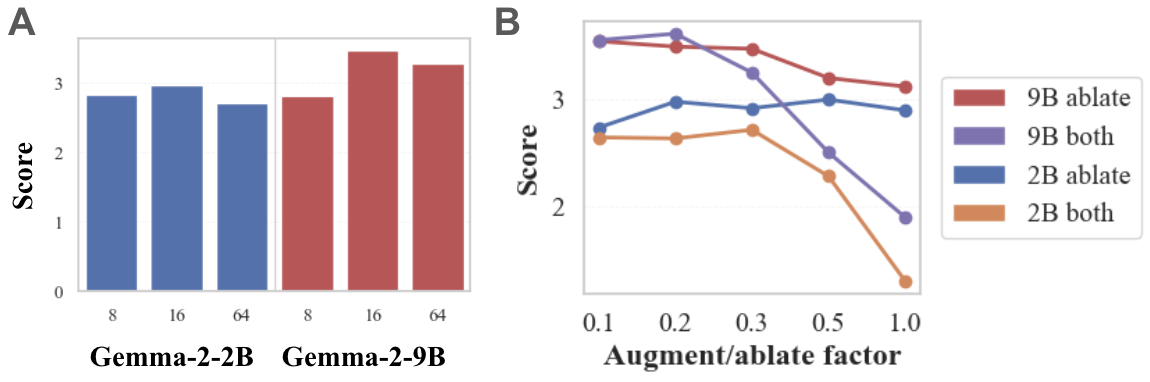}
  \caption{Hyperparameter sensitivity. \textbf{A.} Score vs.\ $k_{\text{diff}}$ (ablate only). \textbf{B.} Score vs.\ $\alpha$ (ablate only; augment+ablate).}
  \label{fig:ablations_kdiff_alpha}
\end{figure}

\paragraph{How many features to steer?} Figure~\ref{fig:ablations_kdiff_alpha}A shows the impact of modifying the parameter $k_{\text{diff}}$, the number of differential features to augment and/or ablate. We find that for both models, among the values surveyed, the best results are obtained for $k_{\text{diff}} = 16$.
\paragraph{How strong should steering be?} In Figure~\ref{fig:ablations_kdiff_alpha}B, we investigate the score obtained for various values of $\alpha$, the factor that controls the steering strength. We observe a sharp decline in performance for high $\alpha$ values when both augmenting and ablating features, whereas there is only a modest decline when ablating only for Gemma-2-9B and little effect at all for Gemma-2-2B. This is related to the fact that we clamp ablated feature values from below by 0, whereas there is no cap for augmented features, allowing them to be overexpressed in the high-$\alpha$ regime. We find that $\alpha = 0.2$ yields the best results for Gemma-2-2B (ablate only) and Gemma-2-9B (augment and ablate) while maintaining high quality for other configurations; this motivates our choice of $\alpha = 0.2$ throughout.

\section{Conclusion}
We introduced DSPA, a prompt-conditional, inference-time preference-alignment method that steers decoding by modifying sparse SAE features via a conditional-difference map from preference data. Across MT-Bench and AlpacaEval on Gemma-2-2B/9B and Qwen3-8B, DSPA is competitive with strong inference-time baselines while preserving multiple-choice accuracy. DSPA remains robust under preference-data restriction and competitive with heavier pipelines at a fraction of the compute. Its sparse intervention traces expose which feature indices are edited at which token positions, and aggregate audits suggest that the preference directions are dominated by discourse and stylistic signals. Future work should test transfer to other safety and long-horizon tasks.

\section*{Acknowledgements}
Aashiq Muhamed acknowledges support from the Amazon AI Ph.D. Fellowship, the Cooperative AI PhD Fellowship, and MATS.

\section*{Limitations}

\paragraph{SAE availability.} DSPA requires high-quality SAEs for both input and output layers, which limits present applicability to model families with suitable SAE infrastructure. Public SAEs remain usable in several tested settings, but preference-domain fine-tuning can substantially improve stronger ablation-based steering on Gemma-2-9B (Section~\ref{sec:ablations}).

\paragraph{Evaluation scope.} Our open-ended generation evaluations rely on LLM-as-a-judge protocols (MT-Bench with GPT-4o, AlpacaEval with Llama-3-70B), which have known biases including preferences for length and style. Judge scores are incomplete proxies for human preference, especially because DSPA often edits discourse features; future work should add controlled human evaluation and length- and style-controlled judge protocols, as well as safety-specific and longer-form benchmarks.

\paragraph{Interpretability pipeline.} Feature interpretations and category assignments are generated by \texttt{gpt-5-mini} without systematic human validation. While the resulting descriptions are plausible and consistent with observed steering behavior, they should be treated as approximate characterizations.

\section*{Ethical Considerations}

\paragraph{Scope and intended use.}
DSPA is an inference-time preference-alignment technique that modifies internal activations using SAE features. It does not provide formal safety guarantees and should not be treated as a substitute for safety training, refusal policies, or content filtering.

\paragraph{Risks and misuse.}
Because DSPA is modular and reversible, it can lower the barrier to experimenting with alignment interventions, but it could also be used to deliberately steer models toward harmful behavior by targeting the same mechanism. We recommend using DSPA only in conjunction with established safety mitigations and validating behavior changes on the intended distribution.

\paragraph{Data and bias.}
DSPA is derived from preference datasets (UltraFeedback, HH-RLHF) and evaluated using LLM-as-a-judge protocols, both of which may encode normative biases about what constitutes a ``good'' response; DSPA may inherit these biases. We do not conduct demographic fairness analyses, and we caution against interpreting benchmark improvements as universally desirable.

\paragraph{Evaluation and interpretability limitations.}
Our open-ended results rely on automated judges and our feature descriptions and category assignments are generated by \texttt{gpt-5-mini}; these signals should not be interpreted as ground truth. We do not evaluate safety-specific benchmarks, so reported gains are benchmark- and judge-dependent.

\paragraph{Use of AI assistants.}
We used \texttt{gpt-5-mini} for feature interpretation and categorization (Section~\ref{sec:interp}, Appendix~\ref{app:interp-prompts}); GPT-4o and GPT-OSS-120B served as MT-Bench judges.

\paragraph{Reproducibility.}
Code and reproduction instructions, including instructions for reproducing the fine-tuned Gemma and Qwen models used in the paper, are available at \url{https://github.com/jtbwedgwood/dspa}. The fine-tuned SAEs are available on Huggingface at links provided in Appendix~\ref{app:base_vs_custom}.

\bibliography{custom}

\appendix

\section{Proofs and Additional Theory}
\label{app:theory}

\subsection{Full Assumptions}

\begin{assumption}[Shared-covariance mean shift, full statement]
\label{ass:lda_full}
There exist a positive semidefinite matrix $\Sigma\in\mathbb R^{d_{\text{SAE}}\times d_{\text{SAE}}}$, a scalar $c>0$, and a prompt-dependent vector $\boldsymbol\beta(x)\in\mathbb R^{d_{\text{SAE}}}$ such that, conditional on $x$,
\begin{align*}
\tilde{\boldsymbol\rho}(y^+)\sim\big(\mu(x)+\tfrac{c}{2}\Sigma\boldsymbol\beta(x),\,\Sigma\big),
\\
\tilde{\boldsymbol\rho}(y^-)\sim\big(\mu(x)-\tfrac{c}{2}\Sigma\boldsymbol\beta(x),\,\Sigma\big),
\end{align*}
for some prompt-dependent mean $\mu(x)$. Equivalently, $\mathbb E[\Delta\tilde{\boldsymbol\rho}\mid x]=c\,\Sigma\boldsymbol\beta(x)$.
\end{assumption}

\begin{assumption}[Additive gating, full statement]
\label{ass:gating_full}
There exists a matrix $B\in\mathbb R^{d_{\text{SAE}}\times d_{\text{SAE}}}$ such that $\boldsymbol\beta(x)=B\,g(x)$. Let $M:=\mathbb E[g(x)\,g(x)^\top]\in\mathbb R^{d_{\text{SAE}}\times d_{\text{SAE}}}$ denote the gate Gram matrix. Write the $i$th column of $B$ as $\boldsymbol\beta^{(i)}$.
\end{assumption}

\subsection{Proof of \cref{thm:identify_exact}}

\begin{proof}
By iterated expectation and \cref{ass:lda},
\begin{align*}
    \mathbf A^\top
= \mathbb E\!\left[\Delta\tilde{\boldsymbol\rho}\,g(x)^\top\right]
= \mathbb E\!\left[\mathbb E[\Delta\tilde{\boldsymbol\rho}\mid x]\,g(x)^\top\right]
= \\ c\,\Sigma\,\mathbb E\!\left[\boldsymbol\beta(x)\,g(x)^\top\right].
\end{align*}
Using \cref{ass:gating}, $\boldsymbol\beta(x)=B g(x)$, so
$\mathbb E[\boldsymbol\beta(x)\,g(x)^\top]=\mathbb E[B g(x)g(x)^\top]=B\,\mathbb E[g(x)g(x)^\top]=B M$,
giving $\mathbf A^\top = c\,\Sigma\,B\,M$. Multiplying by $\Sigma^{-1}$ on the left and $M^{-1}$ on the right yields $\Sigma^{-1}\mathbf A^\top M^{-1} \propto B$.
\end{proof}

\paragraph{Interpretation.}
Equation $\mathbf A^\top = c\,\Sigma\,B\,M$ shows that rows of $\mathbf A$ generally mix multiple templates when gates co-activate: $\mathbf A_{i,:}$ is proportional to $\Sigma\sum_{i'}M_{ii'}\boldsymbol\beta^{(i')}$, not $\Sigma\boldsymbol\beta^{(i)}$ unless $M$ is approximately diagonal.

\subsection{Proof of \cref{cor:weak_coact}}

\begin{proof}
By \cref{ass:lda,ass:gating}, conditioning on $g_i=1$:
\begin{align*}
    \mathbb E[\Delta\tilde{\boldsymbol\rho}\mid g_i=1]
= c\,\Sigma\,\mathbb E[\boldsymbol\beta(x)\mid g_i=1]
= \\ c\,\Sigma\,\mathbb E[B g(x)\mid g_i=1]
= \\ c\,\Sigma\left(\boldsymbol\beta^{(i)}+\sum_{i'\neq i}\pi_{i'\mid i}\boldsymbol\beta^{(i')}\right).
\end{align*}
The norm bound follows from triangle inequality and submultiplicativity:
\[
\|\Sigma \sum_{i'\neq i}\pi_{i'\mid i}\boldsymbol\beta^{(i')}\|_2\le \|\Sigma\|_{2\to 2}\sum_{i'\neq i}\pi_{i'\mid i}\|\boldsymbol\beta^{(i')}\|_2.
\]
If additionally $\|\boldsymbol\beta^{(i)}\|_2>0$ and $\sum_{i'\neq i}\pi_{i'\mid i}\|\boldsymbol\beta^{(i')}\|_2\le \epsilon\|\boldsymbol\beta^{(i)}\|_2$, the relative error is bounded by $\epsilon \cdot \|\Sigma\|_{2\to 2}\|\boldsymbol\beta^{(i)}\|_2 / \|\Sigma\boldsymbol\beta^{(i)}\|_2$.
\end{proof}

\subsection{DSPA Scoring and Optional De-mixing}

At inference time, DSPA constructs an active set $S_{\text{prompt}}$ (e.g., the top $k_{\text{prompt}}$ indices by $\rho_i(x)$) and uses the truncated indicator $\hat g(x):=\mathbf 1_{S_{\text{prompt}}}\in\{0,1\}^{d_{\text{SAE}}}$. The score is
\[
\mathbf s(x)\;:=\;\mathbf A^\top \hat g(x) = c\,\Sigma\,B\,M\,\hat g(x),
\]
where the ideal unmixed preference direction would be $c\,\Sigma\,B\,g(x)$. DSPA makes two approximations: (i) replacing $g(x)$ by a truncated indicator $\hat g(x)$, and (ii) ignoring the off-diagonal co-activation structure encoded in $M$. When $M\approx \mathrm{diag}(p)$ (weak gate dependence), $\mathbf s(x)\approx c\,\Sigma\,B\,\mathrm{diag}(p)\,\hat g(x)$.

\paragraph{Optional de-mixing on the active subspace.}
Although $M$ is too large to invert globally, DSPA only uses a small active set. Let $S:=S_{\text{prompt}}$ and $M_S:=\mathbb E[g_S g_S^\top]\in\mathbb R^{|S|\times |S|}$. If $M_S$ is well-conditioned, a de-mixed score is
\[
\mathbf s_{\mathrm{db}}(x)\;:=\;\mathbf A_{S,:}^\top\,M_S^{-1}\,g_S(x)
\;\approx\; c\,\Sigma\,B_S\,g_S(x).
\]

\subsection{Empirical Diagnostics of the Theory}

We evaluate the de-mixing rule above, diagonal whitening motivated by the covariance term in \cref{ass:lda}, and their combination using the same low-cost MT-Bench proxy as the ablations in Section~\ref{sec:ablations}. Table~\ref{tab:theory_diagnostics} reports the results. None improves on the submitted DSPA score, although most remain above the corresponding Base Model. The corrections therefore provide useful diagnostics but are not empirically preferable to the simpler conditional sparse-feature rule in this setting.

\begin{table*}[t]
\centering
\small
\caption{Theory-motivated variants on the low-cost MT-Bench proxy (higher is better).}
\label{tab:theory_diagnostics}
\begin{tabular}{@{}lccccc@{}}
\toprule
Model & Base & DSPA & De-mixing & Diagonal whitening & De-mixing + whitening \\
\midrule
Gemma-2-2B & 2.61 & \textbf{2.97} & 2.85 & 2.86 & 2.86 \\
Gemma-2-9B & 3.26 & \textbf{3.48} & 3.39 & 3.17 & 3.46 \\
\bottomrule
\end{tabular}
\end{table*}

We also check the two main simplifying assumptions. For the shared-covariance mean-shift approximation in \cref{ass:lda}, the chosen-to-rejected covariance trace ratio is 0.907 and the median ratio over sampled covariance elements is 0.924, supporting similar covariance scale. A corresponding additive-gating diagnostic provides weaker support for \cref{ass:gating}. This partial mismatch is consistent with the corrected variants failing to improve: de-mixing and whitening depend more directly on the idealized assumptions, whereas the submitted top-$k$ heuristic appears more robust to their violation. Accordingly, the theory should be read as a design rationale and diagnostic abstraction, not as a claim that real SAE activations exactly satisfy the model.

\subsection{Optimality of Top-$k$ Ablation}

\begin{assumption}[Small ablation model in density space]
\label{ass:ablation_density}
For a fixed prompt $x$, ablating a set $S\subset[d_{\text{SAE}}]$ of size $k$ produces an expected change $\Delta\tilde{\boldsymbol\rho}(y)\approx -\delta\,\mathbf 1_S$ for some fixed $\delta>0$.
\end{assumption}

\begin{theorem}[Top-$k$ ablation is optimal for linear utility]
\label{thm:topk}
Under \cref{ass:ablation_density} with linear utility $U(x,y)=\boldsymbol\beta(x)^\top \tilde{\boldsymbol\rho}(y)$, the expected utility improvement among all sets $S$ with $|S|=k$ is maximized by choosing the $k$ smallest coordinates of $\boldsymbol\beta(x)$.
\end{theorem}

\begin{proof}
Under \cref{ass:ablation_density}, $\tilde{\boldsymbol\rho}$ shifts by $-\delta\mathbf 1_S$, so $\mathbb E[U(x,y_{\text{edited}})-U(x,y)]=-\delta\sum_{j\in S}\beta_j(x)$. With fixed $k$, this is maximized by selecting the $k$ smallest $\beta_j(x)$.
\end{proof}

\paragraph{Relating $\mathbf s(x)$ to $\boldsymbol\beta(x)$.}
Under \cref{ass:lda}, the mean difference direction is $c\,\Sigma\boldsymbol\beta(x)$, so scoring coordinates by $\mathbf s(x)$ targets (a mixed approximation of) $\Sigma\boldsymbol\beta(x)$. If $\Sigma$ is diagonal with nearly constant entries, rankings are preserved; otherwise, whitening by $\Sigma^{-1}$ before selection may improve results.

\subsection{Finite-Sample Concentration}

Because $\Delta\tilde{\rho}_j\in[-1,1]$ and $g_i\in\{0,1\}$, each summand in $\hat A_{i,j}$ is bounded. Let $N_i:=\sum_{k=1}^N g_i(x_k)$ be the number of training prompts with $g_i(x)=1$.

\begin{lemma}[Row-wise concentration]
\label{lem:concentration}
Conditional on $N_i$, for any $\epsilon>0$,
\[
\Pr\big(|\hat A_{i,j}-A_{i,j}|\ge \epsilon\mid N_i\big)\le 2\exp\!\left(-\tfrac{1}{2}N_i\epsilon^2\right).
\]
By a union bound, with probability at least $1-\delta$,
\[
\max_{j\in[d_{\text{SAE}}]}|\hat A_{i,j}-A_{i,j}|
\le \sqrt{\tfrac{2\log(2d_{\text{SAE}}/\delta)}{N_i}}.
\]
\end{lemma}

\paragraph{Implication.}
Rare gates (small $N_i$) yield noisy conditional estimates, motivating (i) percentile thresholding to control gate frequency and (ii) conservative sparsification of $\mathbf A$ to remove entries dominated by estimation noise.

\section{Fine-Tuned SAEs}
\label{app:base_vs_custom}

\subsection{Fine-Tuning Details and SAE Availability}

For Gemma-2, we initialize from the canonical width-16k JumpReLU residual-stream SAEs released with Gemma Scope \citep{lieberum2024gemmascopeopensparse,rajamanoharan2024jumping}. For Qwen3-8B, we initialize from trainer~0 of the width-16k BatchTopK residual-stream SAEs released by Karvonen \citep{karvonen2025qwen3saes}. We fine-tune each SAE on activations from both the chosen and rejected conversations in HH-RLHF \citep{bai2022traininghelpfulharmlessassistant}, using two passes over packed 512-token sequences and a learning rate of $10^{-5}$ while retaining the source SAE's architecture and latent dimensionality. During Qwen3-8B fine-tuning, we follow Karvonen's recommended preprocessing and discard activation vectors whose norms exceed ten times the median norm within the activation buffer. We train separate SAEs at the input and output intervention layers used for each model; all
checkpoints are available in Table~\ref{tab:finetuned_saes}.

\begin{table}[t]
  \centering
  \small
  \begin{tabular}{llcl}
      \toprule
      Model & Role & Layer & Checkpoint \\
      \midrule
      Gemma-2-2B & Input  & 7  & \href{https://huggingface.co/james-wedgwood/gemma-2-2b-gemma-scope-finetuned-hhrlhf-layer7}
      {Hugging Face} \\
      Gemma-2-2B & Output & 24 & \href{https://huggingface.co/james-wedgwood/gemma-2-2b-gemma-scope-finetuned-hhrlhf-layer24}
      {Hugging Face} \\
      Gemma-2-9B & Input  & 9  & \href{https://huggingface.co/james-wedgwood/gemma-2-9b-gemma-scope-finetuned-hhrlhf-layer9}
      {Hugging Face} \\
      Gemma-2-9B & Output & 39 & \href{https://huggingface.co/james-wedgwood/gemma-2-9b-gemma-scope-finetuned-hhrlhf-layer39}
      {Hugging Face} \\
      Qwen3-8B   & Input  & 9  & \href{https://huggingface.co/james-wedgwood/qwen3-8b-adam-karvonen-finetuned-hhrlhf-layer9}
      {Hugging Face} \\
      Qwen3-8B   & Output & 18 & \href{https://huggingface.co/james-wedgwood/qwen3-8b-adam-karvonen-finetuned-hhrlhf-layer18}
      {Hugging Face} \\
      \bottomrule
  \end{tabular}
  \caption{Fine-tuned SAE checkpoints used in our experiments.}
  \label{tab:finetuned_saes}
\end{table}

\subsection{Analysis: Base vs. Fine-Tuned SAE}

Let $\mathbf{A}^{\text{base}}$ denote the conditional-difference map (Section~\ref{sec:identifying}) computed using the base Gemma Scope SAEs, and $\mathbf{A}^{\text{custom}}$ the same map computed using the custom fine-tuned SAEs, in both cases using Gemma-2-9B with $\ell_{\text{input}}=9, \ell_{\text{output}} = 39$. As an approximation for the features most likely to be augmented or ablated under DSPA, we consider the top $k_{\text{diff}} = 16$ output feature indices $j$ for a given input feature $i$, sorting by $\mathbf{A}_{i,j}$ descending for augment features or ascending for ablate features. Denote by $\mathbf{A}_{i,(k)}$ the $k$th largest and by $\mathbf{A}_{i,(-k)}$ the $k$th smallest element of the row $\mathbf{A}_{i}$; taking the union overall all $i$ yields 
\begin{align*}
    S^{\text{base}}_{\text{augment}} = \bigcup_{i\in[d_{\text{SAE}}]} \{ j : \mathbf{A}^{\text{base}}_{i,j} \geq \mathbf{A}^{\text{base}}_{i,(16)} \}, \\ \quad S^{\text{base}}_{\text{ablate}} = \bigcup_{i\in[d_{\text{SAE}}]} \{ j : \mathbf{A}^{\text{base}}_{i,j} \leq \mathbf{A}^{\text{base}}_{i,(-16)} \}
\end{align*}
and likewise for $S^{\text{custom}}_{\text{augment}}$ and $S^{\text{custom}}_{\text{ablate}}$. We find that $|S^{\text{base}}_{\text{augment}}| = 183, |S^{\text{base}}_{\text{ablate}}| = 186, |S^{\text{custom}}_{\text{augment}}| = 76, |S^{\text{custom}}_{\text{ablate}}| = 83$. Intuitively, the set of features most likely to be modified for the fine-tuned SAE is several times smaller than that for the base SAE, indicating a stronger concentration of features along generalized preference axes in the fine-tuned case.

Interestingly, we find that many of the indices contained in these sets overlap between the base and fine-tuned SAE; concretely, $|S^{\text{base}}_{\text{augment}} \cap S^{\text{custom}}_{\text{augment}}| = 40, |S^{\text{base}}_{\text{ablate}} \cap S^{\text{custom}}_{\text{ablate}}| = 49$. Despite this, the generated interpretations do not always match; for example, feature 10776 is interpreted as \textit{Questions or statements involving illegal or clearly harmful activities} for the fine-tuned SAE, and as \textit{Hedging or clarification language expressing uncertainty or confusion} for the base SAE.

This small number of candidate features for the fine-tuned SAE is the grounds for our assertion in Section~\ref{sec:identifying} that the matrix $\mathbf{A}^{\text{custom}}$ is sparse in practice. Let $\tau = \min\{ |\mathbf{A}^{\text{custom}}_{i,j}| : \mathbf{A}^{\text{custom}}_{i,j} \geq \mathbf{A}^{\text{custom}}_{i,(16)} \text{ or } \mathbf{A}^{\text{custom}}_{i,j} \leq \mathbf{A}^{\text{custom}}_{i,(-16)} \}$. If we set entries of $\mathbf{A}^{\text{custom}}$ to 0 as long as $|\mathbf{A}^{\text{custom}}_{i,j}|<\tau$, then the number of nonzero entries is approximately $0.002d_{\text{SAE}}^2$ and $\mathbf{A}^{\text{custom}}_{i,j}=0$ for all $j\notin S^{\text{custom}}_{\text{augment}} \cup S^{\text{custom}}_{\text{ablate}}$. Generating responses to all first-turn MT-Bench questions and both ablating and augmenting features, we discover that the ablated features are a strict subset of $S^{\text{custom}}_{\text{ablate}}$ and the augmented features are a strict subset of $S^{\text{custom}}_{\text{augment}}$, validating this approximation.

\section{Robustness to Preference Source}
\label{app:preference_sources}

To test whether DSPA is tied to UltraFeedback's prompt distribution, we reconstruct $\mathbf{A}$ from 1,000 filtered PRISM preference pairs \citep{kirk2024prism} and from 1,000 filtered LMSYS Chatbot Arena preference pairs \citep{chiang2024chatbotarena}. We retain the Gemma hyperparameters and the low-cost MT-Bench proxy used in Section~\ref{sec:ablations}. Table~\ref{tab:preference_sources} shows scores comparable to UltraFeedback-DSPA and above the Base Model for both alternative sources and both model sizes. This limited proxy experiment indicates that the mechanism is not unique to UltraFeedback; it does not establish unrestricted transfer across preference distributions.

\begin{table}[t]
\centering
\small
\setlength{\tabcolsep}{3pt}
\caption{DSPA maps from different preference sources ($N{=}1{,}000$; low-cost MT-Bench proxy; higher is better).}
\label{tab:preference_sources}
\begin{tabular}{@{}lcccc@{}}
\toprule
Model & Base & UltraFeedback & PRISM & LMSYS \\
\midrule
Gemma-2-2B & 2.61 & \textbf{2.97} & 2.90 & 2.85 \\
Gemma-2-9B & 3.26 & 3.48 & \textbf{3.55} & 3.43 \\
\bottomrule
\end{tabular}
\end{table}

\section{Training Compute Cost Comparison}
\label{app:compute_cost}

We compare the cost of DSPA matrix construction against the two-stage RAHF-SCIT pipeline \citep{liu2024aligning} under the assumptions used in our restricted-data analysis. Following standard dense-transformer FLOP accounting \citep{brown2020languagemodelsfewshotlearners,chowdhery2022palmscalinglanguagemodeling}, we approximate one forward token by $2P$ FLOPs and one training token (forward + backward) by $6P$ FLOPs for a dense $P$-parameter model, ignoring the smaller attention correction and implementation overhead. We take $P=8\times 10^9$ (an 8B-class model), and for DSPA use conservative estimates prompt/chosen/rejected lengths, setting $p=c=r=1000$.

For DSPA, each preference triple requires three forward-only passes: one over the prompt, one over prompt+chosen, and one over prompt+rejected. Hence
\[
F_{\text{DSPA}}(N)=2P\,N(3p+c+r).
\]
Under our length estimates, this becomes
\[
F_{\text{DSPA}}(N)=8\times 10^{13}N.
\]

For RAHF step 1, the analyzed configuration consumes $4N$ effective training examples (two epochs over duplicated chosen/rejected variants), each with two gradient-tracked passes at effective length $L_1=768$, following the hyperparameter configuration in \citep{liu2024aligning}. This gives
\[
F_{\text{step1}}(N)=48PNL_1=2.95\times 10^{14}N.
\]
For RAHF step 2, each example performs three no-grad forward passes and one gradient-tracked pass, for an effective cost of roughly $12P$ FLOPs/token. With effective optimizer-step batch $B=64$, sequence length $L_2=512$, and $\texttt{max\_steps}=0.02N$,
\[
F_{\text{step2}}(N)=12P(BL_2)(0.02N)=6.29\times 10^{13}N.
\]
Therefore
\begin{align*}
    F_{\text{RAHF}}(N)=F_{\text{step1}}(N)+F_{\text{step2}}(N)= \\3.58\times 10^{14}N,
\end{align*}

so the modeled compute ratio is
\[
\frac{F_{\text{RAHF}}(N)}{F_{\text{DSPA}}(N)}\approx 4.47.
\]

These are model-FLOP estimates rather than measured hardware utilization, so the wall-clock ratio can be larger. In our Gemma-2-9B runs on a single Nvidia H200, DSPA matrix construction took 46 minutes with peak memory 33.1 GB, whereas the full RAHF pipeline took 8 hours 50 minutes with peak memory 140.8 GB. Thus, the observed wall-clock gap was $11.5\times$, substantially larger than the $4.47\times$ modeled FLOP ratio. This extra gap is consistent with implementation overheads that are not captured by the simple FLOP model, including fixed-length padding, gradient checkpointing, and the more complex multi-pass training loop used by RAHF.

\paragraph{DPO comparison.}
Under equal-length assumptions, the same simplified accounting estimates that DPO's alignment stage uses approximately $3.2\times$ the model FLOPs of DSPA construction, before optimizer-state, checkpointing, and memory overhead. Our full-data Qwen3-8B DPO run took nearly 12 hours on four L40S GPUs, whereas DSPA construction completed in roughly one-fifteenth the wall-clock time on one L40S. Hardware utilization and implementations differ, so this is not a controlled throughput benchmark; it illustrates the substantially higher alignment-stage resource cost of the weight-updating baseline used here.

\section{Experimental Setup}
\label{app:experiment}

\subsection{Evaluations}

We obtain MT-Bench results using FastChat (\url{https://github.com/lm-sys/FastChat}), AlpacaEval results using the official repository (\url{https://github.com/tatsu-lab/alpaca_eval}), and MCQ benchmark results using the Language Model Evaluation Harness (\url{https://github.com/EleutherAI/lm-evaluation-harness}). Custom code is used for response and loglikelihood generation for each of these, while the provided code is used for evaluations. We use the common leaderboard settings and metrics for all MCQ benchmarks, as shown in Table~\ref{tab:mcq}.

\begin{table}[H]
\caption{Multiple-choice benchmark settings.}
\label{tab:mcq}
\small
\centering
\begin{tabular}{@{}lcc@{}}
\toprule
\bf Benchmark & \bf Few-shot & \bf Metric \\
\midrule
MMLU & 5 & acc \\
Arc-Easy & 25 & acc\_norm \\
TruthfulQA & 0 & mc2 \\
HellaSwag & 10 & acc\_norm \\
Winogrande & 5 & acc \\
\bottomrule
\end{tabular}
\end{table}

\subsection{Baselines}
\label{app:baselines}

Table~\ref{tab:baselines} lists the hyperparameter configurations for our baselines. For single-layer interventions, we use the same output layer as with DSPA, namely $\ell_{\text{output}} = 24$ for Gemma-2-2B and $\ell_{\text{output}} = 39$ for Gemma-2-9B. For step 2 of RAHF, we run LoRA on layers $\{8, 10, 12, 14\}$ for Gemma-2-2B and $\{ 10, 12, 14, 16, 18 \}$ for Gemma-2-9B, attempting to choose layers analogous to those used in the original paper for the given models. When performing our data ablation experiments, we scale down the max steps parameter on step 2 proportionately to avoid overfitting. Open-ended results are averaged over three generation seeds, and we did not conduct an extensive per-model hyperparameter sweep for training baselines (especially DPO); stronger baseline performance may be achievable with additional tuning.

\begin{table}[H]
\caption{Baseline hyperparameters.}
\label{tab:baselines}
\small
\centering
\begin{tabular}{@{}llc@{}}
\toprule
\bf Baseline & \bf Hyperparameter & \bf Value \\
\midrule
DPO & Epochs & 1 \\
DPO & Learning rate & $5 \times 10^{-6}$ \\
DPO & Beta & 0.1 \\
DPO & Optimizer & AdamW \\
\midrule
RepE & Layer & $\ell_{\text{output}}$ \\
RepE & Directions & 16 \\
RepE & Scale & 0.5 \\
\midrule
Static-SAE & Layer & $\ell_{\text{output}}$ \\
Static-SAE & Scale & 0.5 \\
\midrule
RAHF & Mode & SCIT \\
RAHF (step 1) & Epochs & 64 \\
RAHF (step 1) & Learning rate & $2 \times 10^{-5}$ \\
RAHF (step 1) & Warmup ratio & 0.1 \\
RAHF (step 2) & Learning rate & $3 \times 10^{-4}$ \\
RAHF (step 2) & Max steps & 450 \\
RAHF (step 2) & LoRA rank & 8 \\
RAHF (step 2) & LoRA alpha & 16 \\
RAHF (step 2) & LoRA dropout & 0.05 \\
RAHF (step 2) & Alpha & 5 \\
\bottomrule
\end{tabular}
\end{table}

\section{MT-Bench Category Breakdown}
\label{app:mt_bench}

For more detailed inspection of the strengths and weakness of DSPA relative to our baselines, we provide a breakdown of MT-Bench scores by category across the surveyed methods, shown in Table~\ref{tab:mt_bench_breakdown}. We find that DSPA is competitive across categories. Likely due to adjusting style and tone features, it scores the best on the Roleplay category for both models, while maintaining performance on more objective tasks such as Math and Coding. Indeed, it also performs the best of any method studied on Math for both models, although scores are low across the board there for the small models in question.

\begin{table*}[t]
\caption{MT-Bench scores by category.}
\label{tab:mt_bench_breakdown}
\centering
{\fontsize{8}{9}\selectfont
\setlength{\tabcolsep}{4pt}
\begin{tabular}{@{}cl cccccccc@{}}
\toprule
& & \bf Coding & \bf Extraction & \bf Humanities & \bf Math & \bf Reasoning & \bf Roleplay & \bf STEM & \bf Writing \\
\midrule
\multirow{6}{*}{\rotatebox{90}{\footnotesize Gemma-2-2B}}
& \cellcolor{oursrow} DSPA (ours)   & \cellcolor{oursrow}2.20 & \cellcolor{oursrow}3.20 & \cellcolor{oursrow}5.92 & \cellcolor{oursrow}\bf 2.40 & \cellcolor{oursrow}3.00 & \cellcolor{oursrow}\bf 6.60 & \cellcolor{oursrow}5.72 & \cellcolor{oursrow}\bf 6.05 \\
& DPO           & \bf 3.10 & 4.38 & \bf 6.67 & 2.00 & \bf 3.02 & 4.95 & 5.65 & 5.09 \\
& RepE          & 1.93 & \bf 4.97 & 5.90 & 2.30 & 2.95 & 5.25 & \bf 5.96 & 5.26 \\
& Static-SAE    & 1.48 & 4.00 & 4.60 & 2.15 & 2.10 & 4.58 & 4.49 & 4.50 \\
& Prompt Eng    & 2.05 & 3.20 & 5.90 & \bf 2.40 & 2.15 & 5.03 & 5.72 & 4.05 \\
& Base Model    & 2.30 & 3.80 & 5.90 & 2.20 & 2.90 & 4.05 & 5.92 & 4.17 \\
\midrule
\multirow{6}{*}{\rotatebox{90}{\footnotesize Gemma-2-9B}}
& \cellcolor{oursrow} DSPA (ours)   & \cellcolor{oursrow}\bf 3.55 & \cellcolor{oursrow}5.05 & \cellcolor{oursrow}\bf 6.78 & \cellcolor{oursrow}\bf 2.35 & \cellcolor{oursrow}3.77 & \cellcolor{oursrow}\bf 5.55 & \cellcolor{oursrow}\bf 7.17 & \cellcolor{oursrow}5.35 \\
& DPO           & 3.30 & \bf 6.12 & 5.67 & 2.25 & \bf 3.90 & 5.10 & 6.67 & \bf 6.50 \\
& RepE          & 2.90 & 4.08 & 6.70 & 1.95 & 3.55 & 5.10 & 6.15 & 5.35 \\
& Static-SAE    & 2.15 & 5.15 & 5.53 & 2.00 & 2.95 & 5.10 & 6.50 & 5.25 \\
& Prompt Eng    & 3.10 & 5.10 & 6.55 & 1.85 & 3.27 & 4.68 & 6.53 & 5.70 \\
& Base Model    & 3.33 & 3.90 & 5.86 & 1.95 & 3.45 & 5.12 & 6.55 & 4.92 \\
\bottomrule
\end{tabular}
}
\end{table*}

\section{Personalization Experiment}
\label{app:personalization}

We also tested whether DSPA's interpretable latent interventions can support lightweight personalization. We use Gemma-2-9B with the same output SAE as in the main experiments. We curate 200 prompts containing varied content from the PersonalLLM dataset \citep{zollo2025personalllmtailoringllmsindividual}, with an 80/20 train-test split. On the train set, we measure how ablating each candidate latent changes eight deterministic style axes relative to the base model, as shown in Table~\ref{tab:personalization_latent_axes}. We then construct persona-specific interventions either by selecting the top five latents with the highest estimated utility gain or by ablating all latents with positive estimated gain.

Evaluation is performed on the test set, where we report two complementary scoring protocols. First, we use a deterministic axis rubric: each response is scored on the eight axes and aggregated into a scalar utility under a fixed persona-specific weight vector. Second, we evaluate against the ten reward models released with PersonalLLM \citep{zollo2025personalllmtailoringllmsindividual}, reporting mean judge $z$-scores (Table~\ref{tab:personalization_personalllm}). For the latter method, we test two variants of the DSPA-based method, judge-axis and judge-contrast. Judge-axis DSPA selects latents using the deterministic eight-axis personalization rubric: we estimate how each latent shifts the hand-designed style axes, then rank latents by their predicted utility under a judge-specific preference profile derived from those axes. Judge-contrast DSPA instead selects latents directly from paired reward differences, favoring latents whose ablation most improves responses preferred by a given PersonalLLM judge over responses it disfavors, without going through the intermediate axis decomposition. We test against the following baselines:
\begin{enumerate}
    \item \textbf{Instruction.} For each persona, we provide the model with a short natural-language description of that user's preferences and ask it to answer in a way that matches those preferences, without retrieving any example responses or modifying the model's activations.
    \item \textbf{Few-shot ICL.} Following the most successful intervention from prior work \citep{zollo2025personalllmtailoringllmsindividual}, we prepend one or five persona-matched example prompt-response pairs to the test prompt and then decode normally, without any latent intervention. The pairs are selected from the training set based on cosine similarity to the current prompt.
\end{enumerate}

Across both protocols, latent ablation shows promise as a complementary personalization mechanism, though prompt-based baselines remain strong overall. On the deterministic rubric, DSPA alone (top 5) improves over the base model but trails both instruction-based and few-shot ICL methods; however, combining DSPA with 5-shot ICL yields the best results (3.173), suggesting that latent interventions and retrieval-based methods capture complementary signals. On the reward-model protocol, DSPA-based methods outperform instruction prompting but do not reach the performance of few-shot ICL. These results indicate that DSPA's stylistic interventions can augment other personalization strategies but are not yet a standalone substitute for retrieval-based approaches.

\begin{table}[htb]
\centering
\small
\caption{Personalization: deterministic utility (higher is better).}
\label{tab:personalization_deterministic}
\begin{tabular}{lc}
\toprule
Method & Mean utility \\
\midrule
Base model & 2.400 \\
DSPA (unconstrained) & 2.431 \\
DSPA (top 5) & 2.539 \\
Instruction & 2.573 \\
5-shot ICL & 2.875 \\
5-shot ICL + DSPA (top 5) & 2.888 \\
5-shot ICL + DSPA (unconstrained) & \textbf{3.173} \\
\bottomrule
\end{tabular}
\end{table}

\begin{table}[htb]
\centering
\small
\caption{PersonalLLM reward-model evaluation (mean judge $z$-score over 10 reward models; higher is better).}
\label{tab:personalization_personalllm}
\begin{tabular}{lc}
\toprule
Method & Mean judge $z$-score \\
\midrule
Instruction & -2.239 \\
Judge-axis DSPA (top 5) & -2.128 \\
Judge-contrast DSPA (top 5) & -2.061 \\
1-shot ICL & -1.973 \\
5-shot ICL & \textbf{-1.853} \\
\bottomrule
\end{tabular}
\end{table}

Table~\ref{tab:personalization_latent_axes} shows the five single-latent ablations with the largest total absolute shift across the eight deterministic axes on the train split. The dominant patterns are broad discourse/style effects rather than clean one-axis controls: for example, latent 6345 most strongly shifts conciseness and proactivity, while latent 3491 strongly increases structure. This supports the view that DSPA personalization operates through reusable stylistic features.

\begin{table*}[t]
\centering
\small
\caption{Mean train-set axis shift $\Delta$ after ablating the five most influential latents (positive toward the left side of each axis).}
\label{tab:personalization_latent_axes}
\begin{tabular}{lccccc}
\toprule
Axis & 6345 & 3491 & 1077 & 13585 & 4260 \\
\midrule
Conciseness vs. verbosity & -0.106 &  0.000 & -0.044 &  0.019 &  0.044 \\
Friendliness vs. professionalism &  0.031 & -0.025 &  0.019 &  0.031 &  0.031 \\
Directness vs. diplomacy & -0.056 & -0.081 & -0.062 & -0.100 & -0.069 \\
Structure vs. freeform &  0.100 &  0.219 & -0.050 &  0.050 &  0.000 \\
Task focus vs. relational focus & -0.019 & -0.019 & -0.069 & -0.038 & -0.044 \\
Proactive vs. minimal scope & -0.094 & -0.006 & -0.019 & -0.069 & -0.056 \\
Certainty vs. caveating & -0.044 & -0.075 & -0.075 & -0.056 & -0.081 \\
Creativity vs. literalism &  0.006 & -0.013 & -0.050 &  0.013 & -0.025 \\
\bottomrule
\end{tabular}
\end{table*}

\section{Token-Level Intervention Audit}
\label{app:token_audit}

We instrument the full decoding trace for two MT-Bench examples: the Hawaii travel-blog prompt shown in Figure~\ref{fig:example} (question 81) and an email-template prompt requesting feedback on a quarterly financial report (question 82). Table~\ref{tab:token_audit} reports manually selected transition points at which an ablated feature was active. This heuristic audit demonstrates what DSPA makes directly observable---the edited feature index, its activation, and its token position---while the accompanying interpretations remain approximate and do not by themselves establish causal attribution.

\begin{table*}[t]
\centering
\scriptsize
\setlength{\tabcolsep}{4pt}
\caption{Illustrative token-level audit from two instrumented MT-Bench traces. Text excerpts and transition points were manually selected; interpretations were generated using the pipeline in Appendix~\ref{app:interp-prompts}.}
\label{tab:token_audit}
\begin{tabular}{@{}cp{4.1cm}p{3.3cm}ccp{4.0cm}@{}}
\toprule
QID & DSPA text at transition & Base-model alternative & Feature & Activation & Approximate feature interpretation \\
\midrule
81 & ``Here's an example of a draft blog post: Recently, I had the great pleasure of traveling to the beautiful island of Hawaii.'' & ``I would start by briefly summarizing my trip, ...'' & 6345 & 27.66 & Polite opening phrases that frame what follows \\
81 & ``Hawaii is also home to many unique cultural attractions. The Pearl Harbor Historic Sites ...'' & ``I was lucky enough to experience some wonderful cultural activities ...'' & 4260 & 9.70 & Discourse markers that signal a shift or continuation in thought \\
81 & ``Another popular attraction is the Iao Valley, home to a beautiful 300-foot waterfall ...'' & ``For great shopping, dining, and nightlife, ...'' & 1077 & 39.58 & Introductory phrases that set up a question or topic \\
82 & ``Dear Supervisor, I am reaching out today regarding the attached Quarterly Financial Report.'' & ``Hi there, I am writing to seek your feedback ...'' & 6345 & 18.43 & Polite opening phrases that frame what follows \\
82 & ``I believe this report clearly demonstrates the positive financial performance of our company ...'' & ``The report provides a number of key insights ...'' & 11641 & 31.78 & Sentence-initial grammatical tokens marking basic structure \\
82 & ``In particular, I am interested in your thoughts on the organization and clarity of the report ...'' & ``What specifically are your thoughts on the level of detail ...'' & 3869 & 16.91 & Friendly, polite conversational lead-ins and transitional phrases \\
\bottomrule
\end{tabular}
\end{table*}

As a second, independently constructed description, the single-feature style-axis ablations in Appendix~\ref{app:personalization} associate the most modified features with greater concision and directness and fewer caveats, consistent with the observed output differences. These converging heuristics support inspectability of the intervention trace, not perfect semantic interpretation of every latent.

\section{Feature Interpretation}
\label{app:interp}

\subsection{Prompts}
\label{app:interp-prompts}

We use the following prompt to generate human-readable feature interpretations from \texttt{gpt-5-mini}:
\lstset{
  basicstyle=\ttfamily\small,
  frame=single,
  breaklines=true,
  columns=fullflexible
}

\begin{lstlisting}
System: You are an interpretability researcher focused on overall trends. Use the provided snippets as evidence, and tolerate some counterexamples. Only mark a feature as unreliable if you are genuinely unsure.

User: Latent {latent_idx} in an SAE. Provide a concise explanation (<=2 sentences). If you are genuinely unsure, return reliable=false. In the snippets, tokens wrapped in [[double brackets]] are where the latent is active; use surrounding context to interpret what the latent represents. It is OK if the pattern has some exceptions.
Return JSON with keys: reliable (bool), explanation (string or null), evidence (string).
High-activation snippets: {snippets}
\end{lstlisting}

We again use \texttt{gpt-5-mini} with the following prompt to assign categories to each interpretation. Because the Safety and Intent categories had relatively few members, for clarity, we folded those into Content and Discourse respectively when reporting results in Section~\ref{sec:interp}.
\begin{lstlisting}
You are given a short description of a language feature (an SAE latent interpretation). Assign it to exactly one of the following categories based on its primary function. If multiple categories seem plausible, choose the most dominant one. Return only the category name, with no explanation.

Categories:
Safety - Language related to harm, illegality, sensitive activities, private data, refusal, caution, or policy-related framing (e.g. illegal activities, self-harm, drugs, personal data requests).
Discourse - Language that manages conversational flow or interaction rather than content (e.g. polite lead-ins, acknowledgements, hedging, topic shifts, explanation onsets).
Intent - Language that frames what kind of action or response is being requested or provided (e.g. asking for instructions, step-by-step guidance, advice framing, question vs answer signaling).
Structure - Tokens whose primary role is positional or structural rather than semantic (e.g. sentence-initial markers, clause boundaries, list markers, punctuation-adjacent or formatting tokens).
Grammatical - Closed-class or near-closed-class grammatical function words with minimal standalone meaning (e.g. articles, pronouns, auxiliaries, conjunctions, generic connectors).
Content - Topic-bearing or referential language with substantive semantic meaning (e.g. named entities, domains like politics or health, concrete objects or actions, lexical patterns like words starting with a specific prefix).
Feature description:
{description}
\end{lstlisting}

\subsection{Ablate and Augment Set Features}

Interpretations for the 50 ablate set and 50 augment set features identified in Section~\ref{sec:interp} are shown in Table~\ref{tab:ablate_set} and Table~\ref{tab:augment_set}, respectively.

\begin{table*}[t]

\caption{Ablate-set features (50), with \texttt{gpt-5-mini} interpretations.}
\label{tab:ablate_set}
\small
\centering
\begin{tabular}{@{}lp{12cm}@{}}
\toprule
\bf Feature & \bf Interpretation \\
\midrule
11569 & Filler phrases used to smooth conversation and request clarification \\
10776 & Questions or statements involving illegal or clearly harmful activities \\
6345 & Polite opening phrases that frame what follows \\
11287 & First-person statements expressing personal views or experiences \\
12550 & Short connective phrases linking ideas across clauses \\
1077 & Introductory phrases that set up a question or topic \\
1810 & Brief acknowledgements and generic conversational replies \\
10995 & Common opening phrases that ease into an explanation \\
4260 & Discourse markers that signal a shift or continuation in thought \\
6167 & Phrases that signal an explanation or reasoning is about to follow \\
8646 & General noun phrases referring to people, roles, or categories \\
5681 & Conjunctions commonly used at the start of multi-part statements \\
3107 & Direct address phrases used at the beginning of a reply \\
14500 & Very common function words and basic grammatical glue \\
13585 & Requests asking for instructions or methods to achieve something \\
6782 & Topic-introducing phrases that mark a new segment of thought \\
11172 & Frequently used lead-in phrases starting an explanation \\
15059 & Conversational opening phrases that introduce a question or new topic \\
3491 & Common everyday words and short phrases spanning many unrelated topics \\
15061 & Prefatory phrases that clarify, hedge, or gently frame what follows \\
13208 & Short topical phrases that mark what is being asked about or explained \\
3869 & Friendly, polite conversational lead-ins and transitional phrases \\
11799 & Very common short words used to open questions or continue dialogue \\
5327 & Requests asking for step-by-step instructions or practical guidance \\
5705 & Small grammatical connectors linking clauses or introducing verb phrases \\
3605 & Structural framing phrases that set up how information will be presented \\
15907 & Compact noun or verb phrases carrying the main focus of a statement \\
13819 & Discourse cue words that introduce, reference, or frame content \\
15659 & Signals marking the start of a substantive explanatory continuation \\
3142 & Food, cooking, and hosting-related language and recipe contexts \\
14408 & Short function words and connective fragments near punctuation \\
12655 & Generic introductory phrases that begin an explanation or transition \\
4456 & Engaging, agreeable openings that signal readiness to explain or list \\
15107 & General action- or advice-framing language introducing recommendations \\
15168 & Very common function words with minimal standalone semantic content \\
6909 & Clear explanatory or instructional language answering a question \\
11641 & Sentence-initial grammatical tokens marking basic structure \\
8180 & Clarifying lead-ins that refine or follow up \\
10118 & Short grammatical tokens like pronouns and auxiliaries \\
5688 & Sentence-opening filler with little consistent meaning \\
9580 & Openings of short, direct questions \\
13493 & Small connective words linking clauses \\
437 & Question-framing cue words and auxiliaries \\
5811 & High-level framing words introducing a topic or question \\
8003 & Common pronouns and conjunctions as grammatical glue \\
2842 & Pronouns and short connectors common in conversation \\
4035 & Short determiners, pronouns, and auxiliaries \\
14791 & Generic conversational utility phrases \\
9631 & First-person statements expressing intentions or requests \\
11855 & Connective function words continuing phrases \\
\bottomrule
\end{tabular}

\end{table*}

\begin{table*}[t]
\caption{Augment-set features (50), with \texttt{gpt-5-mini} interpretations.}
\label{tab:augment_set}
\small
\centering
\begin{tabular}{@{}lp{12cm}@{}}
\toprule
\bf Feature & \bf Interpretation \\
\midrule
141 & Common function words and grammatical connectors without semantic content \\
9825 & Clarification and intent-seeking phrases linking questions and follow-ups \\
10075 & Topic-naming or subject-introducing words and phrases \\
1030 & Short list items and connectors enumerating entities or phrases \\
11031 & Affirmative explanatory openings that introduce a helpful continuation \\
8244 & Contractions and possessive markers using apostrophes \\
15634 & Sentence- or clause-initial tokens marking utterance beginnings \\
4891 & Polite request openings introducing questions or actions \\
13301 & Conversational scaffolding words framing dialogue or transitions \\
2969 & Everyday life topics and general practical advice contexts \\
6268 & Clause-initial function words like articles and auxiliaries \\
11779 & Hedging or softening discourse openings before explanations \\
7569 & Turn-initial or sentence-initial position markers \\
6509 & Place names, locations, and location-focused references \\
6586 & Very short replies handling brief or provocative inputs \\
2135 & Prompts ending with a focused keyword or target phrase \\
9304 & Response-leading fragments introducing explanations or suggestions \\
5314 & Question phrasing and self-referential statements expressing personal perspective \\
5049 & Auxiliary verbs and short grammatical continuations at clause boundaries \\
13656 & Introductory interjections and brief discourse markers near turn starts \\
8172 & Location and containment references like "in", "in my area", "address" \\
4263 & Key topical noun phrases defining the main semantic focus \\
12827 & Frequent topical nouns and short subject phrases across contexts \\
866 & Common connectors like "or/of/the" in lists or alternatives \\
6143 & Clause-beginning tokens and common connectors near sentence starts \\
1820 & Very common function words and generic everyday connective tokens \\
8543 & Conversational lead-ins introducing instructions or explanations \\
15344 & High-frequency stopwords and short boundary tokens like "the/from/to" \\
1477 & Concise direct question forms often beginning with wh-words \\
7375 & High-frequency grammatical connectors like "are/this/the/but/is" \\
12366 & Short glue phrases and function words near sentence or turn boundaries \\
8229 & Initial tokens of named entities and multiword noun phrases \\
16271 & Generic conversational openings and short transition phrases \\
14204 & Salient topical keywords and notable content-bearing terms \\
5426 & Informal conversational lead-ins, hedges, and smalltalk-style discourse markers \\
16033 & Polite explanatory framing phrases introducing clarification or help \\
13130 & Introductory phrases that open a new question or topic \\
6532 & Very common function words and generic connective fillers \\
3594 & Tokens marking clause boundaries and new structural segments \\
5416 & Openings that begin or continue an explanatory reply \\
9670 & Common connective words used during explanation or elaboration \\
13329 & Direct question phrasing signaling a request for information \\
2610 & Utterance-initial phrases starting a new topic or inquiry \\
613 & Openings that introduce explanations, lists, or examples \\
7264 & Words referencing the main topic or subject under discussion \\
15264 & Short grammatical connectors like prepositions and determiners \\
3895 & Utterance-initial tokens marking the very start of a clause \\
14998 & Short function words anchoring clause starts or discourse pivots \\
2208 & Sentence beginnings introducing concrete advice or factual responses \\
10754 & Tokens marking the start of a new question or request \\
\bottomrule
\end{tabular}

\end{table*}

\end{document}